\documentclass[acmsmall]{acmart}
\acmJournal{TEAC}
\acmVolume{37}
\acmNumber{4}
\acmArticle{111}
\acmMonth{8}

\usepackage{booktabs} 
\usepackage{diagbox} 
\usepackage{amsmath,amsfonts,amsthm}
\usepackage{empheq}
\usepackage{algorithmic}
\usepackage{arydshln}
\usepackage{xcolor}
\usepackage{xspace}
\usepackage{soul}
\usepackage{paralist}
\usepackage{tikz}
\usetikzlibrary{arrows} 
\usetikzlibrary[patterns]
\usepackage{caption}
\usepackage{subcaption}
\usepackage{multirow}
\usepackage{graphicx}
\usepackage{nicefrac}
\usepackage{comment}
\usepackage[ruled,linesnumbered]{algorithm2e} 
\usepackage{thmtools}
\usepackage{wrapfig}
\usepackage[inline]{enumitem}
\usepackage{etoolbox}

\newtheorem{remark}{Remark}
\newtheorem*{stmnt*}{Statement}
\newcommand{\USW}{\mathtt{USW}}
\newcommand{\NW}{\mathtt{NW}}
\newcommand{\ESW}{\mathtt{ESW}}
\newcommand{\MMS}{\mathtt{MMS}}
\newcommand{\EIT}{\mathtt{EIT}}

\newcommand{\R}{\mathbb{R}}

\newcommand{\Z}{\mathbb{Z}}

\newcommand{\bsub}{matroid rank }
\newcommand{\boxs}{$(0,1)$-\textsc{OXS} }
\newcommand{\calI}{\mathcal{I}}
\newcommand{\calO}{\mathcal{O}}

\newcommand{\bfx}{\boldsymbol{x}} 
\newcommand{\bfy}{\boldsymbol{y}}

\newcommand{\egend}{\hfill $\blacksquare$}
\newcommand{\remend}{\hfill $\blacklozenge$}
\newcommand{\score}{\boldsymbol{s}}

\newcount\Comments  
\newcommand{\mytodo}[2]
{\ifnum\Comments=1
	{\marginpar{\small{\color{#1}	#2}}}\fi}

\SetKwInput{KwInit}{Initialize}

\newcounter{Bew1}
\newcounter{Bew2}
\newcommand{\appendixproof}[2]{
	
	\stepcounter{Bew1}
	\label{L\arabic{Bew1}}
	\gappto{\appendixProofText}{ \stepcounter{Bew2}\label{\arabic{Bew2}} \subsection{#1} #2}
	\vspace{-2pt}
}

\begin{document}

\title{Finding Fair and Efficient Allocations for Matroid Rank Valuations} 
\author{Nawal Benabbou}
\affiliation{
	\institution{Sorbonne Universit\'{e}, CNRS, Laboratoire d’Informatique de Paris 6, France}
	}
	\email{nawal.benabbou@lip6.fr}
\author{Mithun Chakraborty}
\affiliation{
	\institution{University of Michigan, Ann Arbor}
	}
	\email{dcsmc@umich.edu}
\author{Ayumi Igarashi}
\affiliation{
	\institution{National Institute of Informatics, Tokyo, Japan}
	}
	\email{ayumi\_igarashi@nii.ac.jpp}
\author{Yair Zick}
\affiliation{
	  \institution{University of Massachusetts, Amherst, USA}
	}
	\email{yzick@umass.edu}

\begin{abstract}  
In this paper, we present new results on the fair and efficient allocation of indivisible goods to agents whose preferences correspond to {\em matroid rank functions}. This is a versatile valuation class with several desirable properties (such as monotonicity and submodularity), which naturally lends itself to a number of real-world domains. 
We use these properties to our advantage; first, we show that when agent valuations are matroid rank functions, a socially optimal (i.e. utilitarian social welfare-maximizing) allocation that achieves envy-freeness up to one item (EF1) exists and is computationally tractable. We also prove that the Nash welfare-maximizing and the leximin allocations both exhibit this fairness/efficiency combination, by showing that they can be achieved by minimizing any symmetric strictly convex function over utilitarian optimal outcomes. To the best of our knowledge, this is the first valuation function class not subsumed by additive valuations for which it has been established that an allocation maximizing Nash welfare is EF1. 
Moreover, for a subclass of these valuation functions based on maximum (unweighted) bipartite matching, we show that a leximin allocation can be computed in polynomial time. Additionally, we explore possible extensions of our results to fairness criteria other than EF1 as well as to generalizations of the above valuation classes.
\end{abstract}




\maketitle


\section{Introduction}\label{sec:intro}
Suppose that we are interested in allocating seats in courses to prospective students. How should this be done? On the one hand, courses offer limited seats and have scheduling conflicts; on the other, students have preferences over the classes that they take, which must be accounted for. In addition, students might have exogenous constraints, such as a hard limit on the number of classes they may take. Course allocation can be thought of as a problem of allocating a set of {\em indivisible goods} (course slots) to {\em agents} (students). 
One thing that immediately stands out is that the problem of assigning courses to students is very well-structured: students are either willing or unwilling to sign up for a class; this can be thought of as having either a value of $1$ or of $0$ for being assigned a class --- it makes no sense to assign a student to a class they did not sign up for. In addition, if we assign a set of classes $S$ to a student, and they are able to take it, then they would be able to take any subset of $S$ as well. Finally (this is not trivial to prove, but is indeed true), given two sets of feasible course assignments $S,T$ such that $|S|<|T|$, we can find some class  $o\in T$ such that $S\cup \{o\}$ is also a feasible course assignment. Such ``well-behaved'' structures are also known as \emph{matroids}. 
How should we divide goods among agents with subjective valuations? Can we find a ``good'' allocation in polynomial time? Can we exploit the structure of certain problems to efficiently find good allocations?

These questions have been the focus of intense study in the CS/Econ community in recent years; several justice criteria, as well as methods for computing allocations that satisfy them have been investigated. 
Generally speaking, justice criteria fall into two categories: {\em efficiency}  and {\em fairness}. Efficiency criteria are chiefly concerned with lowering some form of \textit{waste}, maximizing some notion of item utilization, or agent utilities. 
For instance, {\em Pareto optimality} (PO) is a popular efficiency concept which ensures that the value realized by no agent can be improved without diminishing that of another agent. 
Fairness criteria require that agents do not perceive the resulting allocation as mistreating them compared to others; for example, one might want to ensure that no agent prefers another agent's assigned bundle (i.e. subset of goods) to her own bundle -- this criterion is known as {\em envy-freeness} (EF) \cite{Foley1967}. However, envy-freeness is not always achievable when items are indivisible: consider a stylized setting, where there is just one course with one seat for which two students are competing; any student receiving this slot would be envied by the other. A simple solution ensuring envy-freeness would be to withhold the seat altogether, not assigning it to either student.
Withholding items, however, violates most efficiency criteria.\footnote{One of the coauthors applied this solution when his children were fighting over a single toy; the method was ultimately deemed unsuccessful.}

As illustrated above and also observed by \citet{budish2011combinatorial}, envy-freeness is not always achievable, even under \emph{completeness}, a very weak efficiency criterion requiring that each item is allocated to some agent. 
However, a less stringent fairness notion --- {\em envy-freeness up to one good} (EF1) --- can be attained. An allocation is EF1 if for any two agents $i$ and $j$, there is some item in $j$'s bundle whose removal results in $i$ not envying $j$. Complete, EF1 allocations always exist for monotone valuations, and in fact, can be found in polynomial time, thanks to the now-classic \textit{envy graph algorithm} due to \citet{lipton2004approximately}.

It is already challenging to individually achieve strong allocative justice criteria; hence, computationally efficient methods that produce allocations satisfying multiple such criteria simultaneously are of particular interest. \citet{caragiannis2016unreasonable} show that when agent valuations are {\em additive} --- i.e. every agent $i$ values its allocated bundle as the sum of values of individual items --- there exist allocations that are both PO and EF1. Specifically, these are allocations that maximize the product of agents' utilities --- also known as the {\em Nash welfare} (MNW). Further work 
\cite{barman2018finding} shows that such allocations can be found in pseudo-polynomial time. 
While encouraging, these results are limited to agents with additive valuations. 
In particular, they do not apply to settings such as the course allocation problem described above (e.g. being assigned two courses with conflicting schedules will not result in additive gain), or other settings we describe later on. In fact, \citet{caragiannis2016unreasonable} left it open whether their result extends to other natural classes of valuation functions, such as the class of submodular valutions.\footnote{\citet{caragiannis2016unreasonable} do provide an instance of two agents with monotone supermodular/subadditive valuations where no allocation is PO and EF1.} 
At present, little is known about other classes of valuation functions --- this is where our work comes in.

\subsection{Our Contributions}\label{sec:contrib}
We focus on monotone submodular valuations with binary (or dichotomous) marginal gains, which are also known as {\em matroid rank valuations} \cite{oxley2011matroid}. 
In this setting, the added benefit of receiving another item is binary and obeys the law of diminishing marginal returns. 
This is equivalent to the class of valuations that can be captured by \textit{matroid} constraints. Matroids are mathematical structures that generalize the concept of linear independence beyond vector spaces; we refer the interested reader to \cite{oxley2011matroid} for further details. 
Matroids have proven to be a versatile framework for describing a variety of problem domains. In our fair allocation domain, each agent has a different matroid constraint over the collection of items, and her value for a bundle is determined by the size of a maximum independent set included in the bundle.

Matroid rank valuations naturally arise in many practical applications, beyond the course allocation problem described above (where students are limited to either approving/disapproving a class). For example, suppose that a government body wishes to fairly allocate public goods to individuals of different minority groups (say, in accordance with a diversity-promoting policy). 
This could apply to the assignment of kindergarten slots to children from different neighborhoods/socioeconomic classes\footnote{see, e.g. \url{https://www.ed.gov/diversity-opportunity}.} or of flats in public housing estates to applicants of different ethnicities \cite{benabbou2019fairness,benabbou2020price}. 
A possible way of achieving group fairness in this setting is to model each minority group as an agent consisting of many individuals: each agent's valuation function is based on {\em optimally matching} items to its constituent individuals; envy naturally captures the notion that no group should believe that other groups were offered better bundles (this is the fairness notion studied by \citet{benabbou2019fairness}). Such assignment/matching-based valuations (known as OXS valuations \cite{paesleme2017gs}) are non-additive in general, and constitute an important subclass of submodular valuations.

Matroid rank functions correspond to submodular functions with binary (i.e. $\{0,1\}$) marginal gains. The binary marginal gains assumption is best understood in context of matching-based valuations described above --- in this scenario, it simply means that individuals either approve or disapprove of items, and do not distinguish between items they approve (we call OXS functions with binary individual preferences \boxs valuations). 
This is a reasonable assumption in kindergarten slot allocation (all approved/available slots are identical), and is implicitly made in some public housing mechanisms; for instance, Singapore housing applicants are required to effectively approve a subset of flats by selecting a block, and are precluded from expressing a more refined preference model. A similar assumption is made in student course selection, where students de-facto approve certain classes by signing up for them (and are thus precluded from expressing more refined preferences). 

What if we further assume that there are exogenous capacity constraints? This is the case in course selection (students may only approve at most a fixed number of classes), and in housing allocation (ethnic minorities in Singapore may only receive a fixed number of flats \cite{benabbou2020price}). 
In addition, imposing certain constraints on the underlying matching problem retains the submodularity of the agents' induced valuation functions: if there is a hard limit due to a \emph{budget} or an exogenous \emph{quota} (e.g. ethnicity-based quotas in Singapore public housing \cite{Parl1989,chua1991race,sim2003public,deng2013publichousing,kim2013singapore,wong2014estimating,benabbou2020price}; socioeconomic status-based quotas in certain U.S. public school admission systems such as Chicago Public Schools \cite{quick2016chicago,schools2017chicago,USedu2017,benabbou2020price} on the number of items each group is able or allowed to receive,
then agents' valuations are {\em truncated} matching-based valuations. Such valuation functions are not OXS, but are still matroid rank functions (i.e. submodularity is preserved). 
Since agents still have binary/dichotomous preferences over items even with the quotas in place, our results apply to this broader class as well.

Using the matroid framework, we obtain a variety of positive existential and algorithmic results on the compatibility of (approximate) envy-freeness with welfare-based allocation concepts.
The following is a summary of our main results (see also Figure~\ref{fig:blocks} and Table~\ref{cmplxty}):

\begin{enumerate}[label=(\alph*)]
	\item \label{maxuswef} For matroid rank valuations, we show that an EF1 allocation that also maximizes the \textit{utilitarian social welfare} or $\USW$ (hence is Pareto optimal) always exists and can be computed in polynomial time by a simple greedy algorithm.
	\item \label{MNWef1} For matroid rank valuations, we show that leximin\footnote{Roughly speaking, a leximin allocation is one that maximizes the realized valuation of the worst-off agent and, subject to that, maximizes that of the second worst-off agent, and so on.} and MNW allocations both possess the EF1 property. 
	\item \label{lmMNW} For matroid rank valuations, we provide a characterization of the leximin allocations; we show that they are identical to the minimizers of \emph{any} symmetric strictly convex function over utilitarian optimal allocations (equivalently, the maximizers of any symmetric strictly concave function over utilitarian optimal allocations). We obtain the same characterization for MNW allocations. 
	\item \label{polytimelmmnw} For \boxs valuations, we show that both leximin and MNW allocations can be computed efficiently. 
\end{enumerate}

\begin{table*}
	\begin{center}
		\begin{tabular}{|c||c|c|c|}
			\hline
			Valuation class & MNW & Leximin & $\max$-$\USW$+EF1\\\hline
			\boxs &poly (Th. \ref{thm:leximin:poly}) & poly (Th. \ref{thm:leximin:poly}) & poly (Th. \ref{thm:bin_all})\\\hline
			matroid rank & poly (BEF20) & poly (BEF20) & poly (Th. \ref{thm:bin_all}; BEF20)\\\hline
		\end{tabular}
	\end{center}
	\caption{Summary of our computational complexity results: ``poly" denotes polynomial; BEF20 refers to \citet{babaioff2020fair} \label{cmplxty}}
\end{table*}

\begin{figure*}
\centering
    \begin{subfigure}[b]{0.4\textwidth}
    \centering
    \includegraphics[width=\textwidth]{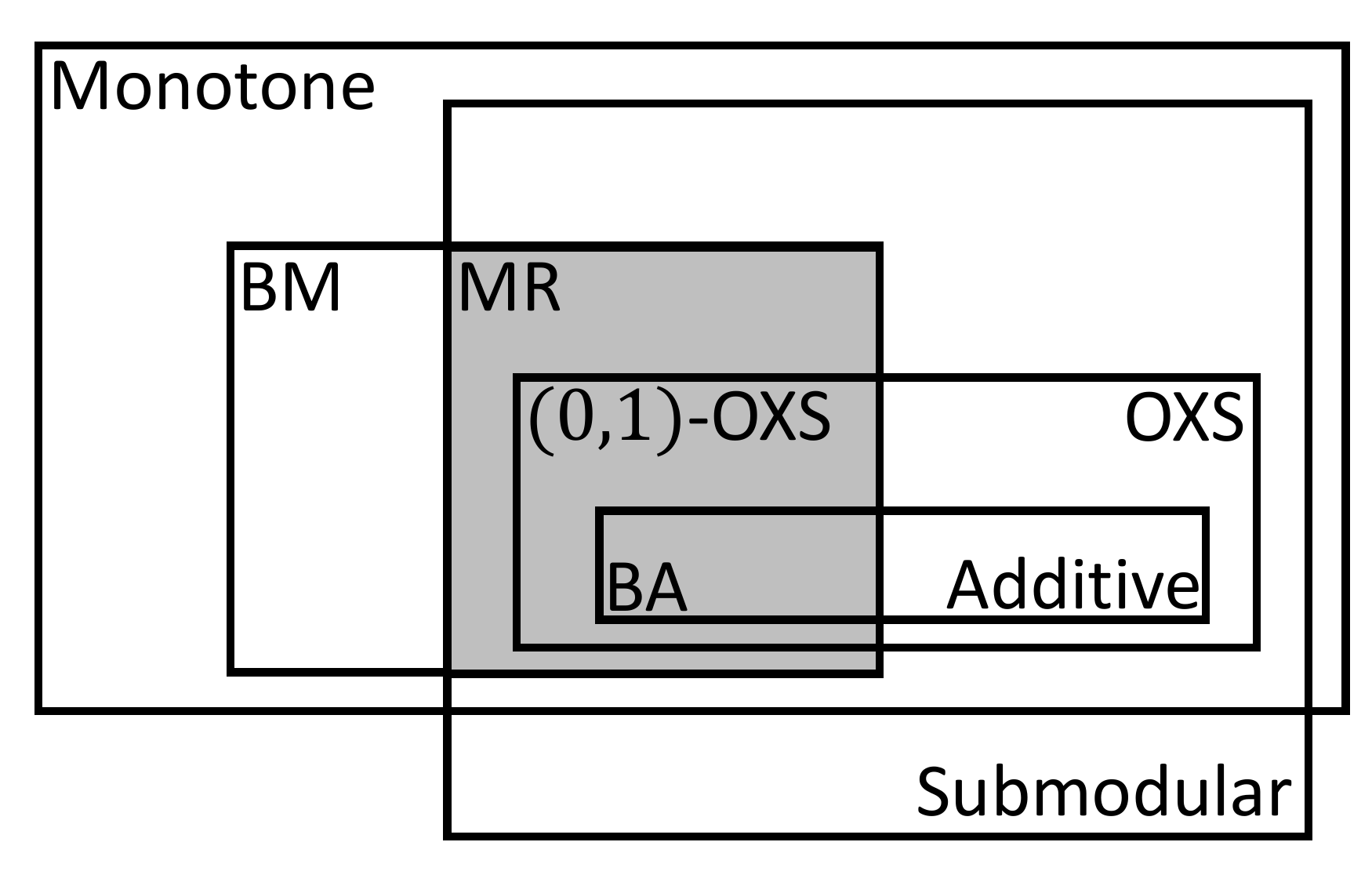}
    \caption{\label{fig:subsump}}
    \end{subfigure}
    \hspace{1cm}
    \begin{subfigure}[b]{0.3\textwidth}
    \centering
    \includegraphics[width=\textwidth]{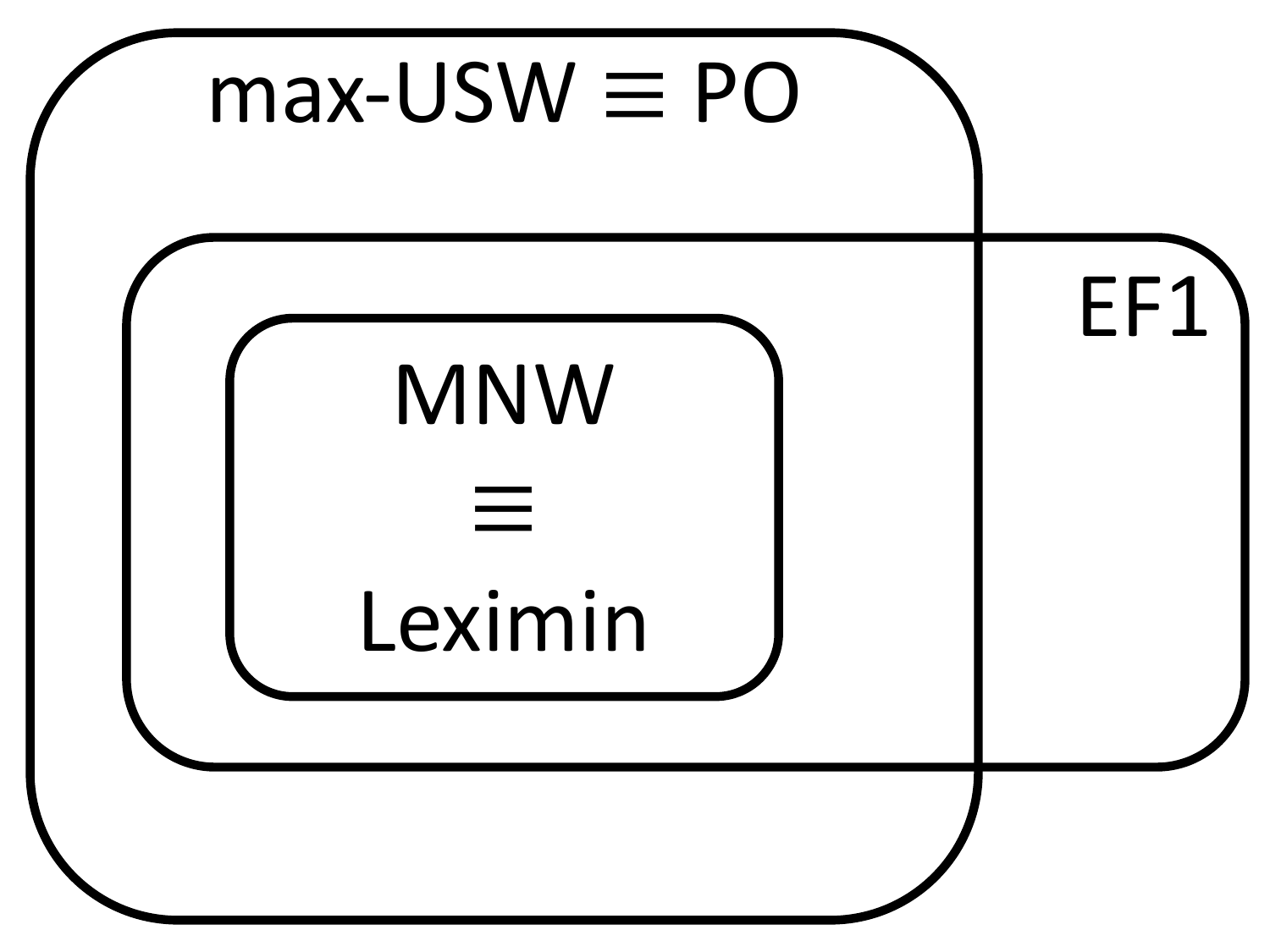}
    \caption{\label{fig:equivs}}
    \end{subfigure}
    \caption{(\protect\subref{fig:subsump}) Subsumption relations and intersections among the various valuation function classes defined in this paper: BM denotes the class of valuation functions with binary marginal gains, MR the matroid rank valuation function class (shaded) --- the class of main interest to us, and BA denotes the class of binary additive valuations. (\protect\subref{fig:equivs}) Equivalences and intersections among properties of clean allocations under the matroid rank valuation class. For both (\protect\subref{fig:subsump}) and (\protect\subref{fig:equivs}), sizes of blocks have no significance. \label{fig:blocks}}    
\end{figure*}


Result~\ref{maxuswef} is remarkably positive: the EF1 and utilitarian welfare objectives are incompatible in general, even for additive valuations, as shown by Example~\ref{ex:ef1notopt} in Appendix~\ref{app:egs}. In fact, maximizing the utilitarian social welfare among all EF1 allocations is NP-hard for general valuations \cite{barman2019fair}.
\appendixproof{Example where no EF1 allocation is utilitarian optimal}{
\begin{example}\label{ex:ef1notopt}
Consider an instance with three items and two agents Alice and Bob having additive valuations as described in Table \ref{table:ef1notopt}. 
\begin{table}[h!]
\begin{center}
\upshape
\setlength{\tabcolsep}{6.4pt}
\scalebox{1}{
\begin{tabular}{rccccc}
\toprule
&
\multicolumn{3}{l}{\!\!\!
\begin{tikzpicture}[scale=0.57, transform shape, every node/.style={minimum size=7mm, inner sep=1.2pt, font=\huge}]	
\node[draw, circle](2)  at (5.4,0) {$1$};
\node[draw, circle](3)  at (6.8,0) {$2$};
\node[draw, circle](4)  at (8.2,0) {$3$};
\end{tikzpicture}\!\!\!\!
\vspace{-2pt}
} \\
\midrule
Alice:\!\! & $\nicefrac{1}{4}$ & $\nicefrac{3}{8}$ &  $\nicefrac{3}{8}$ \\
Bob:\!\! &  $0$ &  $\nicefrac{1}{2}$ & $\nicefrac{1}{2}$   \\
\bottomrule
\end{tabular}
}
\end{center}
\caption{An instance where no EF1 allocation is utilitarian optimal.}
\label{table:ef1notopt}
\end{table}
The unique allocation maximizing $\USW$ is the one which gives item $1$ to Alice and items $2$ and $3$ to Bob for respective valuations $\nicefrac{1}{4}$ and $1$. However, even if either $2$ or $3$ were dropped from Bob's bundle, Alice would value it at $\nicefrac{3}{8} > \nicefrac{1}{4}$; hence, the unique utilitarian optimal allocation is not EF1 and so no EF1 allocation can be utilitarian optimal. \egend
\end{example}
}

Result~\ref{MNWef1} is reminiscent of the Theorem 3.2 in \citet{caragiannis2016unreasonable}, showing that any MNW allocation is PO and EF1 under \emph{additive} valuations; they also showed that a PO and EF1 allocation may not exist under subadditive/supermodular valuations (Theorem 3.3) and MNW does not imply EF1 for arbitrary, real-valued submodular functions (Appendix C Example C.3) but left the PO+EF1 existence question open for the submodular class. 
The open questions in this paper have received substantial attention in recent literature (for instance, progress has been made on EFX or envy-freeness up to the least valued item, see e.g. \cite{plaut2018almost}) but the PO+EF1 existence issue beyond additive valuations is yet to be settled.
To our knowledge, the class of matroid rank valuations is the first valuation class not subsumed by additive valuations for which the EF1 property of the MNW allocation and multiple alternative ways of achieving the PO+EF1 combination have been established. The other properties of the MNW principle that we have uncovered for this valuation class (results~ \ref{MNWef1} and \ref{lmMNW}) may be of independent interest (see the discussion in Section~\ref{sec:disc}). 

Our computational tractability results~\ref{polytimelmmnw} are significant since we know that for arbitrary real valuations, it is NP-hard to compute the following types of allocations: PO+EF even for the seemingly simple class of binary additive valuations which is subsumed by our \bsub class  (\citet{bouveret2008efficiency} Proposition $21$); leximin \cite{bezakova2005allocating}; and MNW \cite{nguyen2014computational}. Moreover, although previous work on binary additive valuations established the polynomial time-solvability of MNW (and thus finding a PO and EF1 allocation) via a clever algorithm based on a subtle running time analysis \cite{barman2018greedy}, we extend this result to the strictly larger \boxs class by uncovering connections to the rich literature on combinatorial optimization. For example, the result \ref{polytimelmmnw} will exploit the network flow technique, drawing the equivalence between leximin allocations and balanced network flows \cite{Murota2018}.

 Our analysis makes extensive use of tools and concepts from matroid theory \cite{oxley2011matroid}. 
 While some papers have explored the application of matroid theory to the fair division problem \cite{BiswasB18,Lauren2017}, we believe that ours is the first to provide with fairness and efficiency guarantees for a matroid-based valuation function class. 

In addition to the above main results~\ref{maxuswef}--\ref{polytimelmmnw}, we present in the appendices a discussion on non-envy-based fairness criteria for matroid rank valuations as well as our attempts at generalizing our results beyond this valuation class:
\begin{itemize}
    \item In Appendix~\ref{sec:subjbin}, we show that for the more general submodular valuations with \textit{subjective} binary marginal gains (for which adding an item to an agent's bundle increases her valuation either by zero or by an item-invariant but agent-specific constant), MNW and leximin allocations no longer coincide -- in particular, an MNW allocation retains the EF1 property whereas leximin allocations do not.
    \item In Appendix~\ref{sec:assval}, we formulate a heuristic extension of the fair allocation algorithm from Result~\ref{maxuswef} above that applies to general assignment (OXS) valuations and evaluate its efficiency in experiments based on a real-world data set \textit{MovieLens-ml-1m} \cite{movieLens}.
    \item In Appendix~\ref{sec:other_fairness}, we explore (approximate) proportionality, equitability, and the maximin share guarantee for matroid rank valuations in particular and submodular valuations in general.
\end{itemize}

\subsection{Related Work}\label{sec:relwork}
\subsubsection{Fair Division with General Valuations}
There is a vast and growing literature on fairness and efficiency issues in resource allocation. Early work on divisible resource allocation provides an elegant result: an allocation that satisfies envy-freeness and Pareto optimality always exists under mild assumptions on valuations \citep{varian}, and can be computed via the convex programming of \citet{eisenberg} for additive valuations. In the domain of the allocation of indivisible goods (see \citet{bouveret2016fair,markakis2017approximation} for an overview), \citet{budish2011combinatorial} was the first to formalize the notion of EF1 as an approximation to envy-freeness; but, it implicitly appears in \citet{lipton2004approximately}. 
More recently, \citet{caragiannis2016unreasonable} prove the discrete analogue of \citet{eisenberg}: MNW allocation satisfies EF1 and Pareto optimality for additive valuations. \citet{barman2018finding} provide a pseudo-polynomial-time algorithm for computing allocations satisfying EF1 and PO. 

Closely related to ours is the work of \citet{BiswasB18} who consider fair division under matroid constraints. In Section $6$ of \citet{BiswasB18}, the authors consider a specific setting where agents have identical additive valuations, and the aim is to find a fair allocation under the same matroid constraint. They show that if a \emph{feasible} allocation satisfying the matroid constraint exists, it can be transformed into an EF1 complete feasible allocation. Although our setting is different from theirs, the proof of Theorem~\ref{thm:bin_all} in our work uses a similar technique to the one used in Theorem $3$ of \citet{BiswasB18} in order to transfer an item from the envious bundle to the envied one.

In this paper, we admit allocations that may be incomplete (i.e. not all items are allocated to the agents under consideration) but satisfy strong fairness and efficiency guarantees. This brings us close to recent work on fairness with ``charity" by \citet{caragiannis2019envy} and \citet{chaudhury2020little}. 

\subsubsection{Binary Additive Valuations}
Under binary preferences, agents either approve or disapprove of each good. 
Due to the advantage of simple elicitation, the use of binary preferences is widespread in social choice literature \cite{LacknerSkowron2020}. In the context of fair division, there has been recent progress on  the study of fair and efficient allocations in this domain. 

In particular, several positive computational results have been obtained for binary additive valuations. \citet{darmann2015nash} and \citet{barman2018greedy} show that the maximum Nash welfare can be computed efficiently for binary additive valuations while computing MNW allocations of indivisible items is hard in general. The work of \citet{barman2018greedy} further develops an efficient greedy algorithm to find an MNW allocation when the valuation of each agent 
is a concave function 
that depends on the number of items approved by her. 
We note that this class of valuations does not subsume the class of \boxs valuations since bundles of the same number of approved items may have different values under the latter class;\footnote{Consider $3$ items, $o_1,o_2,o_3$, and a group of members $S=\{1,2,3\}$ with member $1$ assigning weight $1$ to items $o_1$ and $o_3$, and members $2$ and $3$ assigning weight $1$ to item $o_2$ only. The value of a maximum matching between $\{o_1,o_2\}$ and $S$ is $2$ while the value of a maximum matching between $\{o_1,o_3\}$ and $S$ is $1$.} hence the polynomial-time complexity result of \citet{barman2018finding} does not imply our Theorem \ref{thm:leximin:poly}. 

Independently of our work, \citet{ijcai2020AzizRey} show the equivalence between leximin and MNW in the context of binary additive valuations (Lemma $4$ of \citet{ijcai2020AzizRey}), which is a special case of our result \ref{lmMNW}. 

\citet{halpern2020fair} show that for binary additive valuations, there is a group strategy-proof mechanism that returns an allocation satisfying utilitarian optimality and EF1 (Theorem $1$ of \citet{halpern2020fair}); dropping strategy-proofness, we generalize this result to the class of matroid rank valuations.


\subsubsection{Matroid Rank Valuations}
\citet{babaioff2020fair} independently present a set of results similar to our own; moreover, they explore strategy-proof deterministic and randomized mechanisms for matroid rank valuations, showing that such mechanisms exist. Below, we compare our results with theirs in greater detail. 
\begin{itemize}
    \item \citet{babaioff2020fair} show that a mechanism returning a special MNW allocation (called Prioritized Egalitarian mechanism, or PE mechanism for short) achieves strategy-proofness, EFX$_0$\footnote{We provide the definition of EFX$_0$ in Remark \ref{remark:EFX} in Section \ref{sec:binsub}.}, and maximizes the social welfare for \bsub valuations. In addition, their mechanism runs in polynomial time (Theorem $1$ of \citet{babaioff2020fair}). 
    In this work, we present a simple algorithm that finds a social welfare-maximizing allocation satisfying EF1; however, the returned allocation may not be MNW or EFX$_0$ (see Examples \ref{runningex} and \ref{ex:ef1po_notefx}). In addition, we provide a polynomial-time algorithm returning a MNW allocation for \boxs valuations. 

    \item \citet{babaioff2020fair} show that for binary additive valuations, the PE mechanism achieves MMS (Proposition $5$ of \citet{babaioff2020fair}). They further prove that for general matroid rank functions, although PE mechanism may not satisfy full MMS, it achieves $\frac{1}{2}$MMS (Proposition $6$ of \citet{babaioff2020fair}). We also show that for binary additive valuations, PO and EF1 allocation is MMS (Proposition \ref{prop:MMS}). In addition, we show that there is an instance with \boxs valuations such that no PO and EF1 allocation is MMS. 

    \item \citet{babaioff2020fair} show that for matroid rank functions, PE mechanism minimizes the sum of squares. In this work, we prove a more general statement: for general matroid rank functions, any MNW/leximin minimizes a symmetric convex function of agents' valuations. 

    \item While \citet{babaioff2020fair} do not show the equivalence between MNW and leximin (as we do), they prove a lemma (Lemma $17$ of \citet{babaioff2020fair}) that corresponds to Lemma \ref{lem:sequence} of this work, which is crucial in establishing the equivalence. 

    \item \citet{babaioff2020fair} study the setting where the desirable items are allowed to take an arbitrary value in the range $[1, 1 + \epsilon]$ (called $\epsilon$-leveled valuations), aiming to generalize the positive existence result of a fair and strategy-proof mechanism. Note that the marginal can vary over items so the class of $\epsilon$-leveled valuations strictly generalizes the class of submodular functions with subjective binary marginal gains we study in Appendix \ref{sec:subjbin}. However, there is little overlap with respect to this domain. We show MNW still satisfies EF1 but the equivalence between MNW and leximin is lost for the class of submodular functions with subjective binary marginal gains, which has not been addressed in \citet{babaioff2020fair}. 
\end{itemize}

To conclude, our work was developed independently, and is conceptually different from  \citet{babaioff2020fair} in that the main focus of ours is on the fairness and efficiency compatibility and the properties of such allocations. 


\subsubsection{Other related work}
One motivation for this paper is recent work by \citet{benabbou2019fairness} on promoting diversity in assignment problems through efficient, EF1 allocations of bundles to attribute-based groups in the population. 
Similar works study quota-based fairness/diversity \cite[and references therein]{aziz2019matching,benabbou2020price,suzuki2018efficient}, or by the optimization of carefully constructed functions \cite[and references therein]{ahmed2017diverse,dickerson2019balancing,lang2016multi} in allocation/subset selection.

\section{Model and definitions}\label{sec:prelims}
Throughout the paper, given a positive integer $r$, let $[r]$ denote the set $\{1,2,\dots,r\}$. 
We are given a set $N = [n]$ of {\em agents}, and a set $O = \{o_1,\dots,o_m\}$ of {\em items} or {\em goods}. Subsets of $O$ are referred to as {\em bundles}, and each agent $i \in N$ has a {\em valuation function} $v_i:2^O \to \R_+$ over bundles where $v_i(\emptyset) = 0$.
We further assume polynomial-time oracle access to the valuation $v_i$ of all agents. 
Given a valuation function $v_i:2^O \to \R$, we define the \emph{marginal gain} of an item $o \in O$ w.r.t. a bundle $S \subseteq O$, as $\Delta_i(S;o) \triangleq v_i(S\cup\{o\}) - v_i(S)$. A valuation function $v_i$ is  {\em monotone} if $v_i(S) \le v_i(T)$ whenever $S \subseteq T$. 

An {\em allocation} $A$ of items to agents is a collection of $n$ disjoint bundles $A_1,\dots,A_n$, such that $\bigcup_{i \in N} A_i \subseteq O$; the bundle $A_i$ is allocated to agent $i$. Given an allocation $A$, we denote by $A_0$ the set of unallocated items, also referred to as {\em withheld items}. 
We may refer to agent $i$'s valuation of its bundle $v_i(A_i)$ under the allocation $A$ as its {\em realized valuation} under $A$. 
An allocation is {\em complete} if every item is allocated to \emph{some} agent, i.e. $A_0 = \emptyset$. 
We admit incomplete, but {\em clean} allocations: a bundle $S \subseteq O$ is {\em clean} for $i \in N$ if it contains no item $o \in S$ for which agent $i$ has zero marginal gain (i.e., $\Delta_i(S\setminus \{o\};o)=0$); allocation $A$ is \emph{clean} if each allocated bundle $A_i$ is clean for the agent $i$ that receives it.  
It is easy to `clean' any allocation without changing any realized valuation by iteratively revoking items of zero marginal gain from respective agents and placing them in $A_0$. For example, if for agent $i$, $v_i(\{1\})=v_i(\{2\})=v_i(\{1,2\})=1$, then the bundle $A_i=\{1,2\}$ is not clean for agent $i$ (and neither is any allocation where $i$ receives items $1$ and $2$) but it can be cleaned by moving item $1$ (or item $2$ but not both) to $A_0$.

\subsection{Fairness and Efficiency Criteria}\label{sec:allocative-efficiency}
Our fairness criteria are based on the concept of \emph{envy}. Agent $i$ \emph{envies} agent $j$ under an allocation $A$ if $v_i(A_i) < v_i(A_{j})$.
An allocation $A$ is {\em envy-free} (EF) if no agent envies another.
We will use the following relaxation of the EF property due to \citet{budish2011combinatorial}: we say that $A$ is {\em envy-free up to one good} (EF1) if, for every $i,j \in N$, $i$ does not envy $j$ or there exists $o$ in $A_{j}$ such that $v_i(A_i) \ge v_i(A_{j}\setminus \{o\})$.

The efficiency concept that we are primarily interested in is \emph{Pareto optimality}. 
An allocation $A'$ is said to \emph{Pareto dominate} the allocation $A$ if $v_i(A'_i) \ge v_i(A_i)$ for all agents $i \in N$ and $v_j(A'_j) > v_j(A_j)$ for some agent $j \in N$. 
An allocation is \emph{Pareto optimal} (or PO for short) if it is not Pareto dominated by any other allocation. 

Closely related to the concept of efficiency  is the welfare of an allocation which can be measured in several ways \cite{sen2018collective}. 
Specifically, given an allocation $A$, 
\begin{itemize}
\item its {\em utilitarian social welfare} is $\USW(A) \triangleq \sum_{i = 1}^n v_i(A_i)$; 
\item its {\em egalitarian social welfare} is $\ESW(A) \triangleq\min_{i \in N}v_i(A_i)$; 
\item its {\em Nash welfare} is $\NW(A)\triangleq \prod_{i \in N}v_i(A_i)$. 
\end{itemize}
An allocation $A$ is said to be \emph{utilitarian optimal} (respectively, \emph{egalitarian optimal}) if it maximizes $\USW(A)$ (respectively, $\ESW(A)$) among all allocations. 

Since it is possible that the maximum attainable Nash welfare is $0$ (e.g. if there are fewer items than agents, then one agent must have an empty bundle), we use the following refinement of the maximum Nash social welfare (MNW) criterion used in \cite{caragiannis2016unreasonable}: we find a largest subset of agents, say $N_{\max} \subseteq N$, to which we can allocate bundles of positive values, and compute an allocation to agents in $N_{\max}$ that maximizes the product of their realized valuations. If $N_{\max}$ is not unique, we choose the one that results in the highest product of realized valuations.

The {\em leximin} welfare is a lexicographic refinement of the maximin welfare concept, i.e. egalitarian optimality.
Formally, for real $n$-dimensional vectors $\bfx$ and $\bfy$, $\bfx$ is \textit{lexicographically greater than or equal to}  $\bfy$ (denoted by $ \bfx \geq_{L} \bfy$) if and only if $\bfx=\bfy$, or $\bfx \neq \bfy$ and for the minimum index $j$ such that $x_j \neq y_j$ we have $x_j >y_j$. 
For each allocation $A$, we denote by $\score(A)$ the vector of the components $v_i(A_i)$ $(i \in N)$ arranged in non-decreasing order.
A {\it leximin} allocation $A$ is an allocation that maximizes the egalitarian welfare in a lexicographic sense, i.e., $\score(A) \geq_{L} \score(A')$ for any other allocation $A'$.

\subsection{Submodular Valuations}\label{subsec:defs}
In this paper, agents' valuation functions are not necessarily additive but {\em submodular}. 
A valuation function $v_i$ is {\em submodular} if each single item contributes more to a smaller set than to a larger one, namely, for all $S\subseteq T \subseteq O$ and all $o \in  O\setminus T$, $\Delta_i(S;o) \ge \Delta_i(T;o)$.

One important sub-class of submodular valuations is the class of {\em assignment valuations}. 
This class of valuations was introduced by \citet{shapley1958complements} and is 
synonymous with the OXS valuation class \cite{lehmann2006combinatorial,paesleme2017gs,balcan2012learning}. 
Fair allocation in this setting was explored by \citet{benabbou2019fairness}. 
Here, each agent $h \in N$ represents a group of individuals $N_h$ (such as ethnic groups and genders), each individual $i \in N_h$ (also called a \emph{member}) having a fixed non-negative weight $u_{i,o}$ for each item $o$. An agent $h$ values a bundle $S$ via a \emph{matching} of the items to its individuals (i.e. each item is assigned to at most one member and vice versa) that maximizes the sum of weights \cite{munkres1957algo}; namely,
\[
v_h(S)=\max \{\, \sum_{i \in N_h}u_{i,\pi(i)} \mid \mbox{$\pi \in \Pi(N_h,S)$}\,\}, 
\]
where $\Pi(N_h,S)$ is the set of matchings $\pi:N_h \rightarrow S$ in the complete bipartite graph with bipartition $(N_h,S)$. 

Our particular focus is on submodular functions with {\em binary marginal gains}. 
We say that $v_i$ has {\em binary marginal gains} if $\Delta_i(S;o) \in \{0,1\}$ for all $S\subseteq O$ and $o \in O\setminus S$. 
The class of submodular valuations with binary marginal gains includes the classes of binary additive valuations \cite{barman2018greedy} and of assignment valuations where the weight is binary \cite{benabbou2019fairness}. We say that $v_i$ is a {\em matroid rank} valuation if it is a submodular function with binary marginal gains (these are equivalent definitions \cite{oxley2011matroid}), and \boxs if it is an assignment valuation with binary marginal gains.\footnote{\boxs valuations coincide with rank functions of \textit{transversal} matroids \cite{balcan2012learning}.} The constrained assignment valuations discussed in the fourth paragraph of Section~\ref{sec:contrib} are examples of matroid rank valuations that are not $(0,1)$-OXS. 

\section{Matroid rank valuations}\label{sec:binsub}
The main theme of all results in this section is that, when all agents have \bsub valuations, fairness (EF1) and efficiency (PO) properties are compatible with each other and also with all three optimal welfare criteria we consider. 
Lemma~\ref{lem:opt_po} below shows that Pareto optimality of optimal welfare is unsurprising; but, it is non-trivial to prove the EF1 property in each case (the proof is available in Appendix~\ref{app:egs}).
\begin{lemma}\label{lem:opt_po}
For monotone valuations, every utilitarian optimal, leximin, and MNW allocation is Pareto optimal.
\end{lemma}
\appendixproof{Proof of Lemma~\ref{lem:opt_po}}{\label{app:proof_lem:opt_po}
	\begin{stmnt*}
		For monotone valuations, every utilitarian optimal, leximin, and MNW allocation is Pareto optimal.
	\end{stmnt*}
\begin{proof}
	If an allocation $A$ were not Pareto optimal, then there would be an allocation $A'$ with $v_i(A'_i) >  v_i(A_i)$ for at least one agent $i \in N$ and $v_j(A'_j) \ge  v_j(A_j)$ for every other agent $j \in N \setminus \{i\}$. This implies that $\USW(A') > \USW(A)$ so that $A$ cannot be utilitarian optimal. Moreover, the vector of realized valuations under $A'$ arranged in non-decreasing order must also be lexicographically strictly greater than that of $A$ (since the vectors must differ at the coordinate corresponding to agent $i$, if not earlier), hence $A$ cannot be leximin either. Finally, suppose that $A$ is an MNW allocation if possible, $N_{\max}$ being the subset of agents with strictly positive realized valuations: if $i \in N_{\max}$, then $\prod_{i \in N_{\max}}v_i(A'_i) > \prod_{i \in N_{\max}}v_i(A_i)$, contradicting the optimality of $\NW(A)$, and if $i \in N \setminus N_{\max}$, that would contradict the maximality of $N_{\max}$.
\end{proof}
}

We start the analysis of matroid rank valuations by introducing the basics of matroid theory. 
Formally, a \emph{matroid} is an ordered pair $(E,\calI)$, where $E$
is some finite set and $\calI$ is a family of its subsets (referred to
as the \emph{independent sets} of the matroid), which satisfies the following three axioms: 
\begin{enumerate}[label=(I\arabic*)]
	\item $\emptyset \in \calI$,
	\item if $Y \in \calI$ and $X \subseteq Y$, then $X \in \calI$, and
	\item if $X,Y \in \calI$ and $|X| > |Y|$, then there exists $x \in X \setminus Y$ such that $Y\cup \{x\} \in \calI$.
\end{enumerate}
The rank function $r:2^E \rightarrow \Z$ of a matroid returns the {\em rank} of each set $X$, i.e. the maximum size of an independent subset of $X$. 
Another equivalent way to define a matroid is to use the axiom systems for a rank function.  We require that 
(R1)~$r(X) \leq |X|$, 
(R2)~$r$ is monotone, and 
(R3)~$r$ is submodular. 
Then, the pair $(E, \calI)$ where $\calI=\{\, X \subseteq E \mid r(X)=|X| \,\}$ is a matroid \cite{oxley2011matroid}. In other words, if $r$ satisfies properties (R1)--(R3) then it induces a matroid.

Within the fair allocation context, if an agent has a matroid rank valuation, then the set of \emph{clean} bundles forms the set of independent sets of a matroid. The following are useful properties of matroid rank valuations (the proofs are in Appendix~\ref{app:egs}).
\begin{proposition}\label{prop:binmargprops}
	A valuation function $v_i$ with binary marginal gains is monotone and takes values in $[|S|]$ for any bundle $S$ (hence $v_i(S) \le |S|$).
\end{proposition}
\appendixproof{Proof of Proposition \ref{prop:binmargprops}}{\label{proof_prop:binmargprops}
	\begin{stmnt*}
		A valuation function $v_i$ with binary marginal gains is monotone and takes values in $[|S|]$ for any bundle $S$ (hence $v_i(S) \le |S|$).
	\end{stmnt*}
	\begin{proof}
		Consider subsets of items $T \subset S \subseteq O$ such that $S \setminus T =\{o_1,o_2,\dots,o_r\}$ where $r = |S \setminus T|$. Define $S_0=\emptyset$ and $S_t=\{o_1,o_2,\ldots,o_{t}\}$ for each $t \in [r]$. This gives us the following telescoping series:
		\begin{align*}
		v_i(S) - v_i(T) &=\sum^r_{t=1} (v_i(T \cup S_t) - v_i(T \cup S_{t-1}))\\
		&= \sum^r_{t=1} (v_i(T \cup S_{t-1} \cup \{o_t\}) - v_i(T \cup S_{t-1}))\\ &=\sum^r_{t=1}\Delta_i(T \cup S_{t-1};o_t).
		\end{align*}
		Since all marginal gains are binary, $\Delta_i(T \cup S_{t-1};o_t) \ge 0$ for every $t \in [r]$, hence the above identity implies $v_i(S) - v_i(T) \ge 0$ for $S \supset T$, i.e. $v_i$ is monotone.
		
		Moreover, by setting $T = \emptyset$ and noting that $\Delta_i(T \cup S_{t-1};o_t) \le 1$ for every $t \in [r]$, we get
		$v_i(S) \le v_i(\emptyset)+r = 0 + |S \setminus \emptyset| = |S|$.
	\end{proof}
}

This property leads us to the following equivalence between the size and realized valuation of every \emph{clean} allocated bundle for the \bsub valuation class --- a crucial component of all our proofs. 
Note that cleaning any optimal-welfare allocation leaves the welfare unaltered and ensures that each resulting withheld item is of zero marginal gain to each agent; hence it preserves the PO condition.

\begin{proposition}\label{prop:clean_size}
	For matroid rank valuations, $A$ is a clean allocation if and only if $v_i(A_i)=|A_i|$ for each $i \in N$.
\end{proposition}
\appendixproof{Proof of Proposition \ref{prop:clean_size}}{\label{proof_prop:clean_size}
	\begin{stmnt*}
		For matroid rank valuations, $A$ is a clean allocation if and only if $v_i(A_i)=|A_i|$ for each $i \in N$.
	\end{stmnt*}
	\begin{proof}
	The ``if'' part: Suppose, there is an item $o \in S$ such that $\Delta_i(S \setminus \{o\};o)=0$. Now, by Proposition~\ref{prop:binmargprops}, $v_i(S \setminus \{o\}) \le |S \setminus \{o\}|=|S|-1$ since $o \in S$. This implies that $v_i(S)=v_i(S \setminus \{o\})+\Delta_i(S \setminus \{o\};o) < |S|$. Thus, by contraposition, if $v_i(S)=|S|$, then $\Delta_i(S \setminus \{o\};o)=1$ $\forall o \in S$ since the marginal gain can be either $0$ or $1$, i.e. $S$ is a clean bundle for $i$.
		
	The ``only if'' part: As in the proof of Proposition~\ref{prop:binmargprops}, let $S=\{o_1,o_2,\dots,o_r\}$; define $S_0=\emptyset$ and $S_t=\{o_1,o_2,\ldots,o_{t}\}$ for each $t \in [r]$. By the definition of cleanness, $\Delta_i(S \setminus \{o_t\};o_t)=1$ $\forall t \in [r]$. Since $S_{t-1} \subseteq S \setminus \{o_t\}$ for every $t \in [r]$, 
		$\Delta_i(S_{t-1};o_t) \ge \Delta_i(S \setminus \{o_t\};o_t) = 1$; moreover 
		due to marginal gains in $\{0,1\}$, we must have $\Delta_i(S_{t-1};o_t)=1$ for every $t \in [r]$. Hence, $v_i(S) = \sum_{t=1}^r \Delta_i(S_{t-1};o_t) = r = |S|$.
	\end{proof}
At this point, it is worthwhile to discuss the interplay between cleanness and completeness. It is obvious that every instance admits a utilitarian optimal allocation that is complete (since, given any incomplete utilitarian optimal allocation, we can add withheld items to arbitrary bundles until completeness is reached, keeping the utilitarian social welfare unchanged) and also one that is clean (for analogous reasons). But, it may not be possible to achieve these two properties simultaneously under optimal utilitarian social welfare, as the next result shows. 

\begin{corollary}\label{cor:cleancomp}
	For an instance where all agents have matroid rank valuations, every utilitarian optimal allocation is clean as well as complete if and only if the maximum utilitarian social welfare for the instance under consideration is equal to the number of items. 
\end{corollary}
\begin{proof}
	Since $\sum_{i \in N} |A_i| \le m$ for any allocation $A$, it follows readily from Proposition~\ref{prop:binmargprops} that the maximum utilitarian social welfare of any instance cannot exceed the total number of items $m$ under binary marginal gains.
	
	If an instance admits a utilitarian optimal allocation $A^*$ that is incomplete, i.e. $|A_0|>0$, then $\USW(A^*) = \sum_{i \in N} v_i(A_i) \le \sum_{i \in N} |A_i| < m$. Likewise, if there is an $A^*$ that is complete but not clean, then there is at least one agent $i$ and one item $o \in A^*_i$ such that $\Delta_i(A^*_i\setminus \{o\};o)=0$; hence, $v_i(A^*_i)=v_i(A^*_i\setminus \{o\}) \le |A^*_i\setminus \{o\}| = |A^*_i|-1$. Thus, $\USW(A^*) \le |A^*_i|-1 + \sum_{j \in N\setminus\{i\}}|A^*_j| \le m-1 < m$. Taking the contraposition proves the ``if'' part.   
	
	For the ``only if'' part, note that if an allocation $A$ is complete, then $\sum_{i \in N} |A_i|=m$, and if it is clean under \bsub valuations, then $v_i(A_i)=|A_i|$ for every $i \in N$ by Proposition~\ref{prop:clean_size}. So, if a clean, complete allocation exists under \bsub valuations, then $\USW(A)=m$ which is the highest feasible utilitarian social welfare regardless of the specific valuation functions. Hence, $A$ must be utilitarian optimal with a $\USW$ of $m$.
\end{proof}

A simple example where each utilitarian optimal allocation is either not complete or not clean: $N=\{1,2\}$; $O=\{o_1,o_2,o_3,o_4\}$; $v_1(S) = 1$ and $v_2(S)=\max\{2,|S|\}$ for every $S \in 2^O \setminus \emptyset$. It is easy to see that both $v_1$ and $v_2$ are both \bsub valuations, and the $\USW$ cannot exceed $1+2=3<4$. In any utilitarian optimal allocation, any one item goes to agent $1$ for a valuation of $1$, any two of the remaining items go to agent $2$ for a valuation of $2$, and the final item may be arbitrarily allocated to either agent (not clean) or withheld (incomplete).
}


Example~\ref{ex:envygraph_notPO} in Appendix~\ref{app:lipton_notpo}  
shows that \citet{lipton2004approximately}'s classic envy graph algorithm does not guarantee a Pareto optimal allocation under \bsub valautions (although the output allocation is complete and EF1), and thus underscores the difficulty of finding the PO+EF1 combination under this valuation class. Moreover, note that in the simple example of one good and two agents each valuing the good at $1$, both agents' valuation functions belong to the class under consideration --- this shows that an envy-free and Pareto optimal allocation may not exist even under this class, and further justifies our quest for EF1 and Pareto-optimal allocations. 
\appendixproof{Example showing that \citet{lipton2004approximately}'s algorithm may not produce a Pareto optimal allocation under matroid rank valuations}{\label{app:lipton_notpo}
The algorithm under consideration works as follows: in each iteration, a new item is allocated to an arbitrary agent not currently envied by any other agent; the envy graph is constructed by drawing a directed edge from an agent to every agent it envies; if a cycle forms in the graph, it is eliminated by transferring bundles from envied to envious agent on the (reverse) cycle, starting with the smallest cycle in case of overlapping cycles. We can augment the first step of the above algorithm with a natural heuristic: allocate the item under consideration to an agent that has the maximum marginal gain from it, breaking ties arbitrarily --- for valuations with binary marginal gains, this is equivalent to giving the item to an agent whose marginal gain for it (given its current bundle) is $1$, and to an arbitrary agent if none has non-zero marginal gain for it. 
\begin{example}\label{ex:envygraph_notPO}
Consider $2$ agents and  $2$ items such that $v_1(o_1) = v_1(o_2) = v_1(o_1,o_2)=1$, $v_2(o_1) =v_2(o_1,o_2)=1$ and $v_2(o_2)=0$. \citet{lipton2004approximately}'s algorithm may assign $o_1$ to agent $1$; then $o_2$ will be arbitrarily allocated, resulting in an allocation $A$ with $v_1(A_1)=1$ and $v_1(A_2)=0$. This is Pareto dominated by $A'_1=\{o_2\}$, $A'_2=\{o_1\}$ where each agent realizes a valuation of $1$. The myopic\footnote{The algorithm is \emph{myopic} in the sense that the allocation of a new item in each iteration does not take into account its downstream impact on efficiency and is only geared towards maintaining the EF1 invariant.}, sequential nature of the algorithm results in this undesirable outcome. \egend
\end{example} 
}

\subsection{Finding a Utilitarian Optimal and EF1 Allocation}\label{sec:utiloptef1}
We will now establish that the existence of a PO+EF1 allocation, proved for additive valuations by \citet{caragiannis2016unreasonable}, extends to the class of \bsub valuations. 
In fact, we provide a stronger --- and surprisingly strong --- relation between efficiency and fairness: utilitarian optimality (stronger than Pareto optimality) and EF1 turn out to be mutually compatible under this valuation class. Moreover, such an allocation can be computed in polynomial time!
\begin{theorem}\label{thm:bin_all}
	For \bsub valuations, a utilitarian optimal allocation that is also EF1 exists and can be computed in polynomial time.
\end{theorem}
Our result is constructive: we provide a way of computing the above allocation in Algorithm~\ref{alg_bin}. The proof of Theorem~\ref{thm:bin_all} and those of the latter theorems utilize Lemmas~\ref{lem:revoc_realloc} and~\ref{lem:envy_size} which shed light on the interesting interaction between envy and \bsub valuations. 

\begin{lemma}[Transferability property]\label{lem:revoc_realloc}
For monotone submodular valuation functions, if agent $i$ envies agent $j$ under an allocation $A$, then there is an item $o \in A_j$ for which $i$ has a positive marginal gain with respect to $A_i$. 
\end{lemma}
\begin{proof}
Assume that agent $i$ envies agent $j$ under an allocation $A$, i.e. $v_i(A_i)<v_i(A_j)$, but no item $o \in A_j$ has a positive marginal gain, i.e., $\Delta_i(A_i;o)=0$ for each $o \in A_j$. Let $A_j=\{o_1,o_2,\ldots,o_r\}$. 
As in the proof of Proposition~\ref{prop:binmargprops}, if we define $S_0=\emptyset$ and $S_t=\{o_1,o_2,\ldots,o_{t}\}$ for each $t \in [r]$, we can write the following telescoping series: 
$$v_i(A_i \cup A_j) - v_i(A_i) =\sum^r_{t=1}\Delta_i(A_i \cup S_{t-1};o_t).$$
However, submodularity implies that for each $t \in [r]$,
$\Delta_i(A_i \cup S_{t-1};o_t) \leq \Delta_i(A_i;o_t)=0,$
meaning that 
$$v_i(A_i \cup A_j) - v_i(A_i)=\sum^r_{t=1}\Delta_i(A_i \cup S_{t-1};o_t) = 0.$$ 
Together with monotonicity, this yields $v_i(A_j) \leq v_i(A_i \cup A_j) =v_i(A_i)<v_i(A_j)$, a contradiction. 
\end{proof}
Note that Lemma~\ref{lem:revoc_realloc} holds for submodular functions with arbitrary real-valued marginal gains, and is trivially true for (non-negative) additive valuations. 
However, there exist non-submodular valuation functions that violate the transferability property, even when they have binary marginal gains --- see Example~\ref{ex:nonsubmod} in Appendix~\ref{app:ex:nonsubmod}. 
\appendixproof{Example of a non-submodular valuation function that violates the transferability property}{\label{app:ex:nonsubmod}
\begin{example}\label{ex:nonsubmod}
Agent $1$ wants to have a pair of matching shoes; her current allocated bundle is a single red shoe, whereas agent $2$ has a matching pair of blue shoes. Agent $1$ clearly envies agent $2$, but cannot increase the value of her bundle by taking any one of agent $2$'s items. More formally, suppose $N=[2]$ and $O=\{r_L,b_L,b_R\}$; agent $1$'s valuation function is: $v_1(S) = 1$ only if $\{b_L,b_R\} \subseteq S$, $v_1(S) = 0$ otherwise. Under the allocation $A_1 =\{r_L\}$ and $A_2 = \{b_L,b_R\}$, $v_1(A_1) < v_1(A_2)$ but $\Delta_1(A_1;o) = 0$ for all $o \in A_2$. \egend
\end{example}
}

Below, we show that if $i$'s envy towards $j$ under a clean allocation cannot be eliminated by removing one item from the latter's bundle, then the two agents' valuations for their respective bundles differ by at least two (in fact, we establish a stronger version of the result that does not require the envious agent $i$'s bundle to be clean). 
Formally, we say that agent $i$ envies $j$ up to more than $1$ item if $A_j \neq \emptyset$ and $v_i(A_i) < v_i(A_{j} \setminus \{o\})$ for every $o \in A_j$. 

\begin{lemma}\label{lem:envy_size}
For submodular functions with binary marginal gains, if agent $i$ envies agent $j$ up to more than $1$ item under an allocation $A$ and $j$'s bundle $A_j$ is clean, then $v_j(A_{j}) \ge v_i(A_i)+2$.
\end{lemma}
\begin{proof}
From definition: $A_j \neq \emptyset$ and  $v_i(A_i) < v_i(A_{j} \setminus \{o\})$ for every $o \in A_j$. Consider one such $o$.
From Proposition~\ref{prop:binmargprops}, $v_i(A_j \setminus \{o\}) \leq |A_j \setminus \{o\}|= |A_j| - 1$.  
Since $A_j$ is a clean bundle for $j$, Proposition~\ref{prop:clean_size} implies that $v_j(A_j)=|A_j|$. 
Combining these, we get 
\[
v_i(A_i) < v_i(A_j \setminus \{o\}) \leq |A_j| -1 = v_j(A_j)-1 \qquad \Rightarrow \qquad v_j(A_j) > v_i(A_i)+1,
\] 
which proves the theorem statement since all valuations are integers.
\end{proof}

Next, we show that under \bsub valuations, utilitarian social welfare maximization is polynomial-time solvable (\ref{thm:matroid}). 

\begin{theorem}\label{thm:matroid}
If all agents have submodular functions with binary marginal gains, one can compute a clean utilitarian optimal allocation in polynomial time. 
\end{theorem}
\begin{proof}
We prove the claim by a reduction to the matroid intersection problem. Let $E$ be the set of pairs of items and agents, i.e., $E=\{\, \{o,i\} \mid o \in O \land i \in N \,\}$. 
For each $i \in N$ and $X \subseteq E$, we define $X_i$ to be the set of edges incident to $i$, i.e., $X_i=\{\, \{o,i\} \in X \mid o \in O \,\}$. Note that taking $E=X$, $E_i$ is the set of all edges in $E$ incident to $i \in N$. 
For each $i \in N$ and for each $X \subseteq E$, we define $r_i(X)$ to be the valuation of $i$, under function $v_i(\cdot)$, for the items $o \in O$ such that $\{o,i\} \in X_i$; namely,
\[
r_i(X)=v_i(\{\, o \in O \mid \{o,i\} \in X_i \,\}).
\]  
Clearly, $r_i$ is also a submodular function with binary marginal gains; combining this with Proposition \ref{prop:binmargprops} and the fact that $r_i(\emptyset)=0$, it is easy to see that each $r_i$ is a rank function of a matroid. Thus, the set of clean bundles for $i$, i.e $\calI_i=\{\, X \subseteq E \mid r_i(X)=|X| \,\}$, is the set of independent sets of a matroid. 
Taking the union $\calI=\calI_1 \cup \cdots \cup \calI_n$, the pair $(E,\calI)$ is known to form a matroid \cite{Korte2006}, often referred to as a {\em union matroid}. 
By definition, $\calI = \{\, \bigcup_{i \in N} X_i \mid X_i \in {\calI}_i  \land i \in N \,\}$, so any independent set in $\calI$ corresponds to a union of clean bundles for each $i \in N$ and vice versa. To ensure that each item is assigned at most once (i.e. bundles are disjoint), we will define another matroid $(E,\calO)$ where the set of independent sets is given by
$$
\calO=\{\, X \subseteq E \mid |X \cap E_o| \leq 1, \forall o \in O \,\}.
$$
Here, $E_o=\{\, e=\{o,i\}\mid i \in N\,\}$ for $o \in O$. The pair $(E,\calO)$ is known as a {\em partition matroid} \cite{Korte2006}. 

Now, observe that a common independent set of the two matroids $X \in \calO \cap \calI$ corresponds to a clean allocation $A$ of our original instance where each agent $i$ receives the items $o$ with $\{o,i\} \in X$; indeed, each item $o$ is allocated at most once because $|E_o \cap X| \leq 1$, and each $A_i$ is clean because the realized valuation of agent $i$ under $A$ is exactly the size of the allocated bundle. Conversely, any clean allocation $A$ of our instance corresponds to an independent set $X=\bigcup_{i \in N}X_i \in \calI \cap \calO$, where $X_i = \{\, \{o,i\} \mid o \in A_i \,\}$: for each $i\in N$, $r_i(X_i)=|X_i|$ by Proposition \ref{prop:clean_size}, and hence $X_i \in \calI_i$, which implies that $X \in \calI$; also, $|X \cap E_o| \leq 1$ as $A$ is an allocation, and hence $X \in \calO$. 

Thus, the maximum utilitarian social welfare is the same as the size of a maximum common independent set in $\calI \cap \calO$. It is well known that one can find a largest common independent set in two matroids in time $O(|E|^3 \gamma)$ where $\gamma$ is the maximum complexity of the two independence oracles \cite{Edmonds1979}. 
Since the maximum complexity of checking independence in two matroids $(E,\calO)$ and $(E,\calI)$ is bounded by $O(mnF)$ where $F$ is the maximum complexity of the value query oracle, we can find a set $X \in \calI \cap \calO$ with maximum $|X|$ in time $O(|E|^3 mnF)$. 
\end{proof}

Finally, we are ready to prove Theorem~\ref{thm:bin_all}.
\begin{proof}[Proof of Theorem~\ref{thm:bin_all}]
Consider Algorithm~\ref{alg_bin}. This algorithm maintains optimal $\USW$ as an invariant and terminates on an EF1 allocation. 
Specifically, we first compute a clean allocation that maximizes the utilitarian social welfare. 
The $\EIT$ subroutine in the algorithm iteratively diminishes envy by transferring an item from the envied bundle to the envious agent; Lemma \ref{lem:revoc_realloc} ensures that there is always an item in the envied bundle for which the envious agent has a positive marginal gain.  

\begin{algorithm}
	\DontPrintSemicolon
	\caption{Algorithm for finding utilitarian optimal EF1 allocation}\label{alg_bin}
	Compute a clean, utilitarian optimal allocation $A$.\label{MM}\\					
			   \textbf{/*Envy-Induced Transfers ($\EIT$)*/}\\
				\While{there are two agents $i,j$ such that $i$ envies $j$ more than $1$ item}
						{
						Find item $o \in A_j$ with $\Delta_i(A_i;o)=1$.\\
						$A_j \leftarrow A_j \setminus \{o\}$; $A_i \leftarrow A_i \cup \{o\}$.
						}			
\end{algorithm}

\emph{Correctness}: Each $\EIT$ step maintains the optimal utilitarian social welfare as well as cleanness: an envied agent's valuation diminishes exactly by $1$ while that of the envious agent increases by exactly $1$. Specifically, recall that for matroid rank valuations, an allocation $A$ is clean if and only if $v_i(A_i)=|A_i|$ for all $i \in N$ by Proposition \ref{prop:clean_size}. This means that if the previous allocation $A$ is clean, then we have $v_i(A_i \cup \{o\})=|A_i \cup \{o\}|$, and $v_j(A_j \setminus \{o\})=|A_j \setminus \{o\}|$. Hence the new allocation after each $\EIT$ step remains clean. Thus, if the algorithm terminates, the $\EIT$ subroutine retains the initial (optimal) $\USW$ and, by the stopping criterion, induces the EF1 property. 

To show that the algorithm terminates (in polynomial time), we define the potential function $\Phi(A) \triangleq \sum_{i \in N}v_i(A_i)^2$.
At each step of the algorithm, $\Phi(A)$ strictly decreases by $2$ or a larger integer. To see this, let $A'$ denote the resulting allocation after reallocation of item $o$ from agent $j$ to $i$. Since $A$ is clean, we have $v_i(A'_i)=v_i(A_i)+1$ and $v_j(A'_j)=v_j(A_j)-1$; since all other bundles are untouched, $v_k(A'_k)=v_k(A_k)$ for every $k \in N\setminus \{i,j\}$.
Also, since $i$ envies $j$ up to more than one item under allocation $A$, $v_i(A_i) +2 \leq v_j(A_j)$ by Lemma \ref{lem:envy_size}. 
Combining these, we get
\begin{align*}
\Phi(A')-\Phi(A)&=(v_i(A_i)+1)^2+ (v_j(A_j)-1)^2 -v_i(A_i)^2-v_j(A_j)^2 \\
&=2(1+v_i(A_i)-v_j(A_j))\\
&\leq 2(1-2) = -2. 
\end{align*}

\emph{Complexity}:
By Theorem \ref{thm:matroid}, a clean utilitarian optimal allocation can be computed in polynomial time. The value of the non-negative potential function has a polynomial upper bound: $\sum_{i \in N}v_i(A_i)^2 \leq (\sum_{i \in N}v_i(A_i))^2 \leq m^2$. Thus, Algorithm \ref{alg_bin} terminates in polynomial time. 
\end{proof}

An interesting implication of the above analysis is that a utilitarian optimal allocation that minimizes $\sum_{i \in N}v_i(A_i)^2$ is always EF1. 

\begin{corollary}\label{cor:sumsquares}
	For matroid rank valuations, any clean, utilitarian optimal allocation $A$ that minimizes $$\Phi(A)\triangleq \sum_{i \in N}v_i(A_i)^2$$ among all utilitarian optimal allocations is EF1.
\end{corollary}
\appendixproof{Proof of Corollary \ref{cor:sumsquares}}{\label{app:proof_cor:sumsquares}
	\begin{stmnt*}
		For matroid rank valuations, any clean, utilitarian optimal allocation $A$ that minimizes $\phi(A)\triangleq \sum_{i \in N}v_i(A_i)^2$ among all utilitarian optimal allocations is EF1.
	\end{stmnt*}
	\begin{proof}
		Let $A$ be a clean utilitarian optimal allocation that minimizes the sum of squares of the realized valuations among all utilitarian optimal allocations.
		We will show that $A$ is EF1. Assume towards a contradiction that $A$ is not EF1. Then, there is a pair of agents $i,j$ such that $i$ envies $j$ up to more than $1$ item. By Lemma~\ref{lem:revoc_realloc}, there is an item $o \in A_j$ such that $\Delta_i(A_i;o)=1$. Let $A^*$ be the allocation achieved by transferring $o$ from $j$ to $i$, everything else remaining the same. 
		By Lemma \ref{lem:envy_size} and the fact that $A_j$ is clean, we have
		\[
		v_i(A_i)+2 \leq v_j(A_j), 
		\]
		which implies $\sum_{i \in N} v_i(A^*_i)^2<\sum_{i \in N} v_i(A_i)^2$ proceeding exactly as in the proof of Theorem~\ref{thm:bin_all} --- another contradiction. Hence, $A$ must be EF1. 
	\end{proof}
}

\begin{remark}[Choice of the potential function]\label{rm:potfunc}
	\normalfont	In the proof of Theorem~\ref{thm:bin_all}, we used the sum of squared valuations as the potential function to prove termination in polynomial time mainly for ease of exposition. Additionally, it shows that the $\EIT$ subroutine terminates after $O(m^2)$ iterations. However, any \textit{symmetric, strictly convex,\footnote{See Section~\ref{sec:binsub_mnwlm} for the definition of a symmetric, strictly convex function. For the proof of Theorem~\ref{thm:bin_all}, it suffices for the function to be strictly convex only over the non-negative orthant since valuations are always non-negative.} polynomial} function $\Phi$ of the realized valuations strictly decreases with each $\EIT$ step and, as such, it is sufficient to use any such function as our potential function if we just wish to establish termination in a polynomial number of iterations. Moreover, Corollary~\ref{cor:sumsquares} holds for any such function $\Phi$ as well --- we elaborate on this theme in Section~\ref{sec:binsub_mnwlm}.  \remend
\end{remark}
Despite its simplicity, Algorithm~\ref{alg_bin} significantly generalizes that of \citet{benabbou2019fairness}'s Theorem 4 (which ensures the existence of a non-wasteful EF1 allocation for \boxs valuations) to \bsub valuations. 
We note, however, that the resulting allocation may be neither MNW nor leximin even when agents have \boxs valuations: Example \ref{runningex} below illustrates this and also shows that the converse of Corollary~\ref{cor:sumsquares} does not hold. 
\begin{example}\label{runningex}
The instance we use is identical to Example $1$ in \citet{benabbou2019fairness}. There are two groups (i.e. agents with \boxs valuations) and six items $o_1,o_2,o_3,o_4,o_5,o_6$. The first group $N_1$ (identical to agent $1$) contains four members $a_1,a_2,a_3,a_4$ and the second group $N_2$ (identical to agent $2$) contains four members $b_1,b_2,b_3,b_4$; each individual has utility (weight) $1$ for an item $o$ if and only if she is adjacent to $o$ in the graph depicted in Figure~\ref{fig:Alg1Example}:
\begin{figure}[htb]
\centering
\begin{tikzpicture}[scale=0.85, transform shape, every node/.style={minimum size=6mm, inner sep=1pt}]
	\node[draw, circle,fill=gray!30](1) at (0,2) {$a_1$};
	\node[draw, circle,fill=gray!30](2) at (0,1) {$a_2$};
	\node[draw, circle,fill=gray!30](3) at (0,0) {$a_3$};
	\node[draw, circle,fill=gray!30](4) at (0,-1) {$a_4$};

	\node[draw, circle](11) at (2,2) {$o_1$};
	\node[draw, circle](12) at (2,1) {$o_2$};
	\node[draw, circle](13) at (2,0) {$o_3$};
	\node[draw, circle](14) at (2,-1) {$o_4$};
	\node[draw, circle](15) at (2,-2) {$o_5$};
	\node[draw, circle](16) at (2,-3) {$o_6$};

	\node[draw, circle,fill=gray!30](21) at (4,2) {$b_1$};
	\node[draw, circle,fill=gray!30](22) at (4,1) {$b_2$};
	\node[draw, circle,fill=gray!30](23) at (4,0) {$b_3$};
	\node[draw, circle,fill=gray!30](24) at (4,-1) {$b_4$};
	
	\draw[-, >=latex] (1)--(11) (1)--(16);
	\draw[-, >=latex] (2)--(12) (2)--(14);
	\draw[-, >=latex] (3)--(13);
	\draw[-, >=latex] (4)--(15);
	
	\draw[-, >=latex] (21)--(13);
	\draw[-, >=latex] (22)--(14);
	\draw[-, >=latex] (23)--(15);
	\draw[-, >=latex] (24)--(16);
\end{tikzpicture}
\caption{An instance where Algorithm~\ref{alg_bin} produces an allocation that is not MNW or leximin.\label{fig:Alg1Example}}
\end{figure}
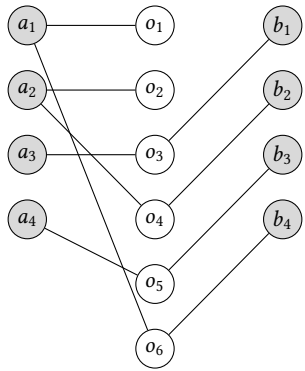

The valuation function of each group for any bundle $X$ is defined as the value (equivalently, the size) of a maximum-size matching of $X$ to the group's members. 
The algorithm may initially compute a utilitarian optimal allocation $A$ that assigns items $o_1,o_2,o_3,o_5$ to the group $N_1$ (with these items assigned to $a_1,a_2,a_3,a_4$ respectively), and the remaining items to group $N_2$ (with $o_4,o_6$ assigned to $b_2,b_4$ respectively). Then, $v_1(A_1)=4>2=v_1(A_2)$ and $v_2(A_2)=2=v_2(A_1)$, hence the allocation $A$ is EF1 --- in fact, envy-free! So, the $\EIT$ subroutine will not be invoked and the output of Algorithm~\ref{alg_bin} will be $A$. However, the (unique) leximin and MNW allocation assigns items $o_1,o_2,o_3$ to the first group, and the remaining items to the second group -- this is also the (unique) utilitarian optimal allocation with the minimum sum of squares of the agents' valuations. \egend
\end{example}

\begin{remark}[EFX allocation]\label{remark:EFX}
	\normalfont	It is worthwhile at this point to comment on the implications of our results for a stronger version of the EF1 property that has received considerable attention in recent literature: envy-freeness up to \emph{any} item, often called the EFX condition. There are two definitions in the literature: 
	\begin{enumerate}
		\item \label{efxplus} \citet{caragiannis2016unreasonable} who introduced this concept (for additive valuations) called it \emph{envy-freeness up to the least (positively) valued good}; we can naturally extend their definition to general valuations as follows: an allocation $A$ is EFX if, for every pair of agents $i,j \in N$ such that $i$ envies $j$, $v_i(A_i) \ge v_i (A_j \setminus \{o\})$ for every item $o \in A_j$ satisfying $\Delta_i(A_j \setminus \{o\}) > 0$. We will call this property EFX$_+$ for clarity (see Remark~\ref{rem:efxplus} in Appendix~\ref{app:rem:efxplus} for issues with defining EFX$_+$ beyond additive valuations). 
	\item \label{efx0} \citet{plaut2018almost} defined an allocation $A$ to be EFX if, for every pair of agents $i,j$, $v_i(A_i) \ge v_i (A_j \setminus \{o\})$ $\forall o \in A_j$ or equivalently $v_i(A_i) \ge \max_{o \in A_j} v_i (A_j \setminus \{o\})$ --- this stronger definition favors allocations where more agents are envy-free of others since $v_i(A_j \setminus \{o\})=v_i(A_j)$ whenever $o$ is of zero marginal value to agent $i$ with respect to the bundle $A_j$: the authors show that under this definition, no EFX allocation can be Pareto optimal even for two agents with additive valuations or general but identical valuations.\footnote{However, both their examples establishing negative results for these sets of conditions on the valuation functions involve eliminating  items with zero marginal value; their second example (for identical valuations) uses a non-submodular valuation function.} \citet{caragiannis2019envy} and \citet{chaudhury2020little} use this definition as well. Following \citet{kyropoulou2019almost}, who studied both the (above) weaker and (this) stronger variants of approximate envy-freeness under a different valuation model, we call this stronger property EFX$_0$.
	\end{enumerate}             
	For \bsub valuations, all items with non-zero marginal values for an agent are also valued identically at $1$, hence EF1 trivially implies EFX$_+$; Theorem~\ref{thm:bin_all} and Corollary~\ref{cor:sumsquares} further guarantee the existence of an EFX$_+$ and PO allocation for any instance under this valuation class. However, we demonstrate with  Example~\ref{ex:ef1po_notefx} with \boxs valuations in Appendix~\ref{app:ex:ef1po_notefx} that even an EF1 and utilitarian optimal (hence PO) allocation may not satisfy the EFX$_0$ condition. 
	\remend
\end{remark}
\appendixproof{Example of an EF1 and utilitarian optimal allocation violating EFX$_0$}{\label{app:ex:ef1po_notefx}
\begin{example}\label{ex:ef1po_notefx}
	 There are two groups and four items $o_1,o_2,o_3,o_4$. The first group $N_1$ has two members $a_1,a_2$ and the second group $N_2$ has three $b_1,b_2,b_3$; each individual has utility $1$ for an item if and only if she is adjacent to it in the graph in Figure~\ref{fig:not_efx0}:
	\begin{figure}[htb]
		\centering
		\begin{tikzpicture}[scale=0.85, transform shape, every node/.style={minimum size=6mm, inner sep=1pt}]
		\node[draw, circle,fill=gray!30](1) at (0,2) {$a_1$};
		\node[draw, circle,fill=gray!30](2) at (0,1) {$a_2$};

		\node[draw, circle](11) at (2,2) {$o_1$};
		\node[draw, circle](12) at (2,1) {$o_2$};
		\node[draw, circle](13) at (2,0) {$o_3$};
		\node[draw, circle](14) at (2,-1) {$o_4$};

		\node[draw, circle,fill=gray!30](21) at (4,2) {$b_1$};
		\node[draw, circle,fill=gray!30](22) at (4,1) {$b_2$};
		\node[draw, circle,fill=gray!30](23) at (4,0) {$b_3$};

		\draw[-, >=latex] (1)--(11) (1)--(12);
		\draw[-, >=latex] (2)--(13);

		\draw[-, >=latex] (21)--(12);
		\draw[-, >=latex] (22)--(13);
		\draw[-, >=latex] (23)--(14);
		\end{tikzpicture}
		\caption{An instance that admits an EF1, utilitarian optimal allocation that is not EFX$_0$\label{fig:not_efx0}}
	\end{figure}
	
	The \boxs valuation functions of groups $N_1$ and $N_2$ are denoted by $v_1(\cdot)$ and $v_2(\cdot)$ respectively. The allocation $A$ where $A_1=\{o_1\}$ and $A_2=\{o_2,o_3,o_4\}$ is utilitarian optimal; it is also EF1 since $v_1(A_1)=1=v_1(A_2 \setminus \{o\})$ for $o \in \{o_2,o_3\}$, with $\Delta_1(A_2\setminus \{o_4\};o_4)=0$, and $v_2(A_2)=3>0=v_2(A_1)$. $A$ could be the output of Algorithm~\ref{alg_bin} and is clean and complete. However, $A$ is not EFX$_0$ since $v_1(A_2 \setminus \{o_4\}) = 2 > 1 = v_1(A_1)$.
	\egend
\end{example}
}
\appendixproof{Remark on the definition of EFX$_+$ for general valuations}{\label{app:rem:efxplus}
\begin{remark}[EFX$_+$ allocation]\label{rem:efxplus}
	\normalfont In this remark, we will point out an issue with extending the definition of envy-freeness up to the least positively valued good to settings with general (possibly non-additive) valuation functions. Recall that in Remark~\ref{remark:EFX} in Section~\ref{sec:binsub}, we defined an allocation $A$ of indivisible items to be EFX$_+$ if for every pair of agents $i$,$j$ such that $i$ envies $j$, $v_i(A_i) \ge v_j(A_j \setminus {o})$ \textit{for every} $o \in A_j$ such $i$'s marginal valuation of $o$ given $A_j$ is \textit{strictly positive}, i.e. $\Delta_i(A_j \setminus \{o\}) > 0$. This definition can make certain non-EF1 allocations trivially EFX$_+$, even under matroid rank valuations, if there is no single item in the envied agent's bundle for which the envious agent has positive marginal valuation. Consider two agents with matroid rank valuations, $v_1(S) = \min\{|S|,2\}$ and $v_2(S) = |S|$, and any allocation $A$ where agent $1$ gets one item and agent $2$ three items. Agent $2$ does not envy agent $1$ but $v_1(A_2 \setminus \{o\}) = 2 > 1 = v_1(A_1)$ for every $o \in A_2$, hence the allocation is not EF1. However, the condition for EFX$_+$ is vacuously satisfied since $\Delta_i(A_j \setminus \{o\}) = 0$ for every $o \in A_2$. This is inconsistent with the property that EFX is stronger than (i.e. implies) EF1, which holds for additive valuations. Note that, as we argued in Remark~\ref{remark:EFX}, EF1 implies EFX$_+$ under matroid rank valuations; we just showed with the above example that the converse is not true, implying that EF1 is a stronger property than EFX$_+$ for this valuation class. This is problematic since the EFX property was originally introduced as a strengthening of EF1 and it is reasonable to want to retain this relation beyond additive valuations! 

A possible way of fixing this issue is by requiring that, for an allocation $A$ to be EFX$_+$ under general valuations, it must first be EF1, i.e. if $i$ is envious of $j$ under $A$, then $i$ must be EF1 of $j$ and additionally the envy of $i$ towards $j$ can be eliminated by removing $i$'s least positively marginally valued item from $j$'s bundle. However, more fundamentally, this issue calls into question the applicability of EFX$_+$ as a fairness concept for non-additive valuations.

On the other hand, EFX$_0$ is stronger than EF1 for any monotone valuation $v(\cdot)$ with $v(\emptyset)=0$. Example~\ref{ex:ef1po_notefx} in Appendix~\ref{app:ex:ef1po_notefx} already demonstrates that an EF1 allocation may not be EFX$_0$ even under \boxs valuations. Conversely, if an allocation $A$ is EFX$_0$, then for any pair of agents $i$ and $j$, there is always an item $o \in A_j$ regardless of $\Delta_i(A_j \setminus \{o\})$ such that $v_i(A_i) \ge v_i(A_j \setminus \{o\})$, i.e. EFX$_0$ always implies EF1.
\remend
\end{remark}
}

Also, see Remark~\ref{rem:wthldset} in Appendix~\ref{app:rem:wthldset} for connections to fair allocation ``with charity" \cite{caragiannis2019envy,chaudhury2020little}.
\appendixproof{Remark on cleanness and fair allocation ``with charity"}{\label{app:rem:wthldset}
\begin{remark}[The withheld set]\label{rem:wthldset}
	\normalfont A remark about the withheld set is in order. The output $A$ of Algorithm~\ref{alg_bin} is clean by construction and hence may be incomplete (Corollary~\ref{cor:cleancomp}), leaving us with a non-empty withheld set $A_0$. If completeness is not a requirement (i.e. there is free disposal and/or no stipulation of the form "All items must go!"), we can view the withheld set as ``surplus'' that can be set aside for future use or ``charity.'' This is similar in spirit to recent work on \emph{EFX-with-charity} \cite{caragiannis2019envy} and \emph{EFX-with-bounded-charity} \cite{chaudhury2020little}. 
	
	For additive valuations, \citet{caragiannis2019envy} design an algorithm that computes an EFX$_0$, \emph{partial} allocation with at least half the optimal Nash welfare, running in polynomial time if it has oracle access to an MNW allocation; \citet{chaudhury2020little} consider general valuations and provide a way to set aside a charity bundle of size strictly less than the number of agents such that the allocation of the remaining items to the agents is EFX$_0$ and every agent weakly prefers her own bundle to the charity bundle --- this allocation retains the high Nash welfare guarantee of \citet{caragiannis2019envy} for additive valuations.   
	
	It is easy to identify similar desirable properties in the output $A$ of Algorithm~\ref{alg_bin}: the (potentially partial) allocation to agents has optimal utilitarian social welfare and is EF1 (although not necessarily EFX$_0$); utilitarian optimality further dictates that every agent $i \in N$ weakly prefers its allocated bundle to the withheld set, i.e. $v_i(A_i) \ge v_i(A_0)$, and also has zero marginal value for any subset of withheld items given its bundle, i.e. $v_i(A_i \cup S) = v_i(A_i)$ $\forall S \subseteq A_0$; $|A_0|=m-U^*$, where $U^*$ is the optimal social welfare which is equal to the number of items allocated under any clean, utilitarian optimal allocation (by Proposition~\ref{prop:clean_size}) .
	
	Although we did not start with the problem of withholding items to satisfy desiderata, this outcome emerged from our search for EF1+PO allocations beyond additive valuations. \remend
\end{remark}
}

\subsection{MNW and Leximin Allocations}\label{sec:binsub_mnwlm}

We saw in Section~\ref{sec:utiloptef1} that under \bsub valuations, a simple iterative procedure allows us to reach an EF1 allocation while preserving utilitarian optimality. 
However, as we previously noted, such allocations are not necessarily leximin or MNW. In this subsection, we characterize the set of leximin and MNW allocations under \bsub valuations. 
We start by showing that Pareto optimal allocations coincide with utilitarian optimal allocations when agents have \bsub valuations. 
Intuitively, if an allocation is not utilitarian optimal, one can always find an `augmenting' path that makes at least one agent happier but no other agent worse off. 

In the subsequent proof, we will use the following notions and results from matroid theory: Given a matroid $(E,\calI)$, the sets in $2^E \setminus \calI$ are called {\em dependent}, and a minimal dependent set of a matroid is called a {\em circuit}. The following is a crucial property of circuits. 
\begin{lemma}[\citet{Korte2006}]\label{lem:circuit}
Let $(E,\calI)$ be a matroid, $X \in \calI$, and $y \in E \setminus X$ such that $X \cup \{y\} \notin \calI$. Then the set $X \cup \{y\}$ contains a unique circuit.
\end{lemma}
Given a matroid $(E,\calI)$, we denote by $C({\calI},X,y)$ the unique circuit contained in $X \cup \{y\}$ for any $X \in \calI$ and $y \in E \setminus X$ such that $X \cup \{y\} \not \in \calI$. 

\begin{theorem}\label{thm:PO}
For \bsub valuations, any Pareto optimal allocation is utilitarian optimal. 
\end{theorem}
\begin{proof}
Define $E$, $X_i$, $E_i$, $\calI_i$ for $i \in N$, $\calI$, and $\calO$ as in the proof of Theorem \ref{thm:matroid}. 
We first observe that for each $X \in \calI$ and each $y \in E \setminus X$, if $X \cup \{y\} \not \in \calI$, then there is agent $i \in N$ whose corresponding items in $X_i$ together with $y$ is not clean, i.e., 
$X_i \cup \{y\} \not \in \calI_i$,
which by Lemma \ref{lem:circuit} implies that the circuit $C({\calI},X,y)$ is contained in $E_i$, i.e., 
\begin{equation}\label{eq:circtui}
C({\calI},X,y)=C({\calI_i},X,y).
\end{equation} 

Now to prove the claim, let $A$ be a Pareto optimal allocation. Without loss of generality, we assume that $A$ is clean. 
Then, as we have seen before, $A$ corresponds to a common independent set $X^*$ in $\calI \cap \calO$ given by
\[
X^{*} =\bigcup_{i \in N} \{\, e=\{o,i\} \in E \mid o \in A_i \,\}. 
\]
Suppose towards a contradiction that $A$ does not maximize the utilitarian social welfare. This means that $X^{*}$ is not a largest common independent set of $\calI$ and $\calO$. 
It is known that given two matroids and their common independent set, if it is not a maximum-size common independent set, then there is an `augmenting' path \cite{Edmonds1979}. 

To formally define an augmenting path, we define an auxiliary graph $G_{X^{*}}=(E,B^{(1)}_{X^{*}}\cup B^{(2)}_{X^{*}})$ where the set of arcs is given by
\begin{align*}
&B^{(1)}_{X^{*}}= \{\, (x,y)\mid y \in E \setminus X^{*} \land x \in C({\calO}, X^{*},y)\setminus \{y\} \,\},\\
&B^{(2)}_{X^{*}}= \{\, (y,x)\mid y \in E \setminus X^{*} \land x \in C({\calI}, X^{*},y)\setminus \{y\} \,\}.
\end{align*}
Since $X^{*}$ is not a maximum common independent set of $\calO$ and $\calI$, the set $X^{*}$ admits an \emph{augmenting} path, which is an alternating path $P=(y_0,x_1,y_1,\ldots,x_s,y_s)$ in $G_{X^{*}}$ with $y_0,y_1,\ldots,y_s \not \in X^*$ and $x_1,x_2,\ldots,x_s \in X^*$, 
where $X^{*}$ can be augmented by one element along the path, i.e., 
\[
X'=(X^{*} \setminus \{x_1,x_2,\ldots,x_s\})\cup \{y_0,y_1,\ldots,y_s\} \in \calI \cap \calO.
\]
Now let's write the pairs of agents and items that correspond to $y_t$ and $x_t$ as follows:  
\begin{itemize}
\item $y_t=\{i(y_t),o(y_t)\}$ where $i(y_t) \in N$ and $o(y_t) \in O$ for $t=0,1,\ldots,s$; and 
\item $x_t=\{i(x_t),o(x_t)\}$ where $i(x_t) \in N$ and $o(x_t) \in O$ for $t=1,2,\ldots,s$.
\end{itemize}
Since each $x_t$ $(t \in [s])$ belongs to the unique circuit $C({\calI},X^{*},y_{t-1})$, which is contained in the set of edges incident to $i(y_{t-1})$ by the observation made in \eqref{eq:circtui}, we have $i(x_t)=i(y_{t-1})$ for each $t \in [s]$. This means that along the augmenting path $P$, each agent $i(x_{t})$ receives a new item $o(y_{t-1})$ and discards the old item $o(x_{t})$. 

Now consider the reallocation corresponding to $X'$ where agent $i(x_t)$ receives a new item $o(y_{t-1})$ but loses the item $o(x_{t})$ for each $t=1,2,\ldots,s$, and agent $i(y_s)$ receives the item $o(y_s)$. Such a reallocation increases the valuation of agent $i(y_s)$ by $1$, while it does not decrease the valuations of all the intermediate agents, $i(x_1),i(x_2), \ldots, i(x_{s})$, as well as the other agents whose agents do not appear on $P$. We thus conclude that $A$ is Pareto dominated by the new allocation, a contradiction.
\end{proof}

Theorem~\ref{thm:PO} above, along with Lemma~\ref{lem:opt_po}, implies that both leximin and MNW allocations are utilitarian optimal. Next, we show that for the class of \bsub valuations, leximin and MNW allocations are identical to each other; further, they can be characterized as the minimizers of any symmetric strictly convex function among all utilitarian optimal allocations. 

A function $\Phi:\Z^n \rightarrow \R$ is {\em symmetric} if for any permutation $\pi:[n] \rightarrow [n]$,
\begin{align*}
&\Phi(z_1,z_2,\ldots,z_n)=\Phi(z_{\pi(1)},z_{\pi(2)},\ldots,z_{\pi(n)}),  
\end{align*}
and is {\em strictly convex} if for any $\bfx,\bfy \in \Z^n$ with $\bfx \neq \bfy$ and $\lambda \in (0,1)$ where $\lambda \bfx +(1-\lambda) \bfy$ is an integral vector, 
\begin{align*}
&\lambda \Phi(\bfx)+(1-\lambda)\Phi(\bfy) > \Phi(\lambda \bfx + (1-\lambda) \bfy). 
\end{align*}
A function $\Psi:\Z^n \rightarrow \R$ is {\em strictly concave} if for any $\bfx,\bfy \in \Z^n$ with $\bfx \neq \bfy$ and $\lambda \in (0,1)$ where $\lambda \bfx +(1-\lambda) \bfy$ is an integral vector, 
\begin{align*}
&\lambda \Psi(\bfx)+(1-\lambda)\Psi(\bfy) < \Psi(\lambda \bfx + (1-\lambda) \bfy). 
\end{align*}
It is not difficult to see that $\Phi:\Z^n \rightarrow \R$ is strictly convex if and only if $-\Phi$ is strictly concave. 
Examples of symmetric, strictly convex functions are the following: 
\begin{align*}
\Phi(z_1,z_2,\dots,z_n) &\triangleq \sum_{i=1}^n z_i^2 &\text{for $z_i \in \Z$\quad $\forall i \in [n]$;}\\
\Phi(z_1,z_2,\dots,z_n) &\triangleq \sum_{i=1}^n z_i \ln z_i &\text{for $z_i \in \Z_{\ge 0}$ $\forall i \in [n]$.}
\end{align*}
For an allocation $A$, we define $\Phi(A) \triangleq \Phi(v_1(A_1),v_2(A_2),\dots,v_n(A_n))$.

We start by showing that given a non-leximin socially optimal allocation $A$, there exists an adjacent socially optimal allocation $A’$ which is the result of transferring one item from a `happy' agent $j$ to a less `happy' agent $i$. The underlying submodularity guarantees the existence of such allocation. We denote by $\chi_i$ the $n$-dimensional incidence vector where the $j$-th component of $\chi_i$ is $1$ if $j=i$, and it is $0$ otherwise.
\begin{lemma}\label{lem:sequence}
Suppose that agents have matroid rank valuations. Let $A$ be a utilitarian optimal allocation. If $A$ is not a leximin allocation, then there is another utilitarian optimal allocation $A'$ such that 
\[
\score(A')=\score(A)+\chi_i-\chi_j, 
\]
for $i,j \in [n]$ with $\score(A)_j \geq \score(A)_i+2$. 
\end{lemma}
\begin{proof}
Let $A$ be an arbitrary utilitarian optimal allocation which is not leximin, and let $A^*$ be a leximin allocation. Recall that $A^*$ is utilitarian optimal by Theorem \ref{thm:PO}. Without loss of generality, we assume that both $A$ and $A^*$ are clean allocations. Now take a clean allocation $A'$ that minimizes the symmetric difference $\sum_{i \in N}|A'_i \triangle A^*_i|$ over all clean allocations with $\score(A')=\score(A)$. 
Assume also w.l.o.g. that $v_1(A'_1) \leq v_2(A'_2) \leq \cdots \leq v_n(A'_n)$. We let $v_{j_1}(A^*_{j_1}) \leq v_{j_2}(A^*_{j_2}) \leq \cdots \leq v_{j_n}(A^*_{j_n})$. Since $A^*$ lexicographically dominates $A'$, for the minimum index $k$ with $v_k(A'_k) \neq v_{j_k}(A^*_{j_k})$, 
\begin{equation}\label{eq:k}
v_k(A'_k) < v_{j_k}(A^*_{j_k}). 
\end{equation}
We note that $v_h(A'_h) = v_{j_h}(A^*_{j_h})$ for all $1 \leq h \leq k-1$. By \eqref{eq:k}, there exists $i \in [k]$ with 
\begin{equation}\label{eq:i}
v_i(A'_i) < v_i(A^*_i). 
\end{equation}
Indeed, if for all $i \in [k]$, $v_i(A'_i) \geq v_i(A^*_i)$, the $k$-th smallest value of realized valuations under $A'$ is at least $v_{j_k}(A^*_{j_k})$, contradicting with \eqref{eq:k}. Take the minimum index $i$ satisfying \eqref{eq:i}. Since both $A'$ and $A^*$ are clean allocations, we have 
\begin{equation}\label{eq:i:size}
|A'_i|=v_i(A'_i) < v_i(A^*_i)=|A^*_i|.
\end{equation} 
By minimality, for all $h \in [i-1]$, $v_h(A'_h) \geq v_h(A^*_h)$. In fact, the equality 
\begin{equation}\label{eq:h}
v_h(A'_h) = v_h(A^*_h)
\end{equation} 
holds for all $h \in [i-1]$. Indeed if $v_h(A'_h) > v_h(A^*_h)$ for some $h \in [i-1]$, then $h$-th smallest value of the realized valuations under $A'$ would be strictly greater than that under $A^*$, yielding $\score(A')>_L\score(A^*)$, a contradiction. 

Now, recall that the family of clean bundles $\calI_h=\{\, S \subseteq O \mid v_h(S)=|S|\,\}$ for $h \in N$ forms a family of independent sets of a matroid. By $($I3$)$ of the independent-set matroid axioms and by the inequality \eqref{eq:i:size}, there exists an item $o_1 \in A^*_i \setminus A'_i$ with positive contribution to $A'_i$, i.e., $v_i(A'_i \cup \{o_1\})=v_i(A'_i)+1$. 
By utilitarian optimality of $A'$, $o_1$ is allocated to some agent, i.e., $o_1 \in A'_{i_1}$ for some $i_1 \neq i$. Consider the following three cases: 
\begin{itemize}
\item Suppose $v_{i_1}(A'_{i_1}) \geq v_i(A'_i)+2$. Then, we obtain a desired allocation by transferring $o_1$ from $i_1$ to $i$. 
\item Suppose $v_{i_1}(A'_{i_1})=v_i(A'_i)+1$. Then by transferring $o_1$ from $i_1$ to $i$, we get another utilitarian optimal allocation with the same vector as $\score(A')$, which has a smaller symmetric difference than $\sum_{i \in N}|A'_i \triangle A^*_i|$, a contradiction. 
\item Suppose $v_{i_1}(A'_{i_1}) \leq v_i(A'_i)$. We will first show that $v_{i_1}(A'_{i_1}) \leq v_{i_1}(A^*_{i_1})$. By \eqref{eq:h}, this clearly holds if $i_1 \leq i$. Also, when $i_1 > i$, this means that $v_{i_1}(A'_{i_1}) = v_i(A'_i)$; thus $v_{i_1}(A'_{i_1}) \leq v_{i_1}(A^*_{i_1})$, as otherwise the $i$-th smallest value of realized valuations under $A'$ would be greater than that under $A^*$, contradicting that $A^*$ is leximin. Further by the facts that $|A'_{i_1} \setminus \{o_1\}| < |A^*_{i_1}|$ and that both $A'_{i_1} \setminus \{o_1\}$ and $A^*_{i_1}$ are clean (i.e., independent sets of a matroid), there exists an item $o_2 \in A^*_{i_1} \setminus A'_{i_1}$ such that $v_{i_1}(A'_{i_1}\cup \{o_2\} \setminus \{o_1\})=v_i(A'_{i_1})$. Again by utilitarian optimality of $A'$, $o_2$ is allocated to some agent, i.e., $o_2 \in A'_{i_2}$ for some $i_2 \neq i_1$. 
\end{itemize}
Repeating the same argument and letting $i_0=i$, we obtain a sequence of items and agents $(i_0,o_1,i_1,o_2,i_2,\ldots,o_t,i_t)$ such that 
\begin{itemize}
\item $v_{i_h}(A'_{i_h})=v_{i_h}(A'_{i_h}\cup \{o_{h+1}\} \setminus \{o_{h}\})$ for all $1 \leq h \leq t-1$; and 
\item $o_h \in A^*_{i_{h-1}} \setminus A'_{i_h}$ for all $1 \leq h \leq t$. 
\end{itemize}
See Figure \ref{fig:path} for an illustration of the sequence. If the same agent appears again, i.e., $i_h=i_{h'}$ for some $h<h' \leq t$, then by transferring items along the cycle, we can decrease the symmetric difference with $A^*$, a contradiction. Thus, the sequence must terminate when we reach the agent $i_t$ with $v_{i_t}(A'_{i_t}) \geq v_i(A'_i)+2$. Exchanging items along the path, we get a desired allocation. 
\begin{figure}[htb]
		\centering
		\begin{tikzpicture}[scale=0.85, transform shape, every node/.style={minimum size=6mm, inner sep=1pt}]
		
		\draw(0,0) circle (1cm and 1cm);
		\draw(3,0) circle (1cm and 1cm);
		\draw(6,0) circle (1cm and 1cm);
		\draw(12,0) circle (1cm and 1cm);
		
		\node[draw, circle,fill=gray!30](1) at (3,0) {$o_1$};
		\node[draw, circle,fill=gray!30](2) at (6,0) {$o_2$};
		\node at (9,0) {$\cdots$};
		\node[draw, circle,fill=gray!30](t) at (12,0) {$o_t$};
		
		\node at (0,0.6) {$A_{i_0}$};
		\node at (3,0.6) {$A_{i_1}$};
		\node at (6,0.6) {$A_{i_2}$};
		\node at (12,0.6) {$A_{i_t}$};
		
		\draw[->, >=latex] (1)--(1,0);
		\draw[->, >=latex] (2)--(4,0);
		\draw[->, >=latex] (t)--(10,0);
		\end{tikzpicture}
		\caption{The path $(i_0,o_1,i_1,o_2,i_2,\ldots,o_t,i_t)$}
		\label{fig:path}
	\end{figure}
\end{proof}

We further observe that such adjacent allocation decreases the value of any symmetric strictly convex function (equivalently, increases the value of any symmetric strictly concave function). The proof is similar to that of Proposition $6.1$ in \citet{Murota2018}, which shows the analogous equivalence over the integral base-polyhedron.

\begin{lemma}\label{lem:convex}
Let $\Phi:\Z^n \rightarrow \Z$ be a symmetric strictly convex function and $\Psi:\Z^n \rightarrow \Z$ be a symmetric strictly concave function. Let $A$ be a utilitarian optimal allocation. Let $A'$ be another utilitarian optimal allocation such that $\score(A')=\score(A)+\chi_i-\chi_j$ for some $i,j \in [n]$ with $\score(A)_j \geq \score(A)_i+2$. Then, $\Phi(A)>\Phi(A')$ and $\Psi(A) < \Psi(A')$. 
\end{lemma}
\begin{proof}
Let $\beta=\score(A)_j-\score(A)_i \geq 2$, and $\bfy=\score(A)+\beta(\chi_i -\chi_j)$. Thus $\Phi(\score(A))=\Phi(\bfy)$ by symmetry of $\Phi$. Define $\lambda = 1 - \frac{1}{\beta}$. We have $0<\lambda<1$ since $\beta \geq 2$. Observe that  
\begin{align*}
\lambda \score(A) + (1-\lambda) \bfy &=(1 - \frac{1}{\beta}) \score(A) + \frac{1}{\beta} (\score(A)+\beta(\chi_i -\chi_j))\\
&=\score(A)+\chi_i-\chi_j=\score(A'),
\end{align*}
which gives us the following inequality (from the strict convexity of $\Phi$): 
$
\Phi(\score(A))=\lambda \Phi(\score(A)) + (1-\lambda) \Phi(\score(A)) >\Phi(\score(A')).
$
Since $-\Psi$ is symmetric strictly convex, the analogous proof shows $\Psi(A) < \Psi(A')$.
\end{proof}

Now we are ready to prove the following.

\begin{theorem}\label{thm:convex}
Let $\Phi:\Z^n \rightarrow \R$ be a symmetric strictly convex function, and $\Psi:\Z^n \rightarrow \R$ be a symmetric strictly concave function. Let $A$ be some allocation. For \bsub valuations, the following statements are equivalent:
\begin{enumerate}
\item $A$ is a minimizer of $\Phi$ over all the utilitarian optimal allocations; and
\item $A$ is a maximizer of $\Psi$ over all the utilitarian optimal allocations; and
\item $A$ is a leximin allocation; and
\item $A$ maximizes Nash welfare.
\end{enumerate}
\end{theorem}
\begin{proof}
To prove $1 \Leftrightarrow 2$, let $A$ be a leximin allocation, and let $A'$ be a minimizer of $\Phi$ over all the utilitarian optimal allocations. 
We will show that $\score(A')$ is the same as $\score(A)$, which, by the uniqueness of the leximin valuation vector and symmetry of $\Phi$, proves the theorem statement. 

Assume towards a contradiction that $\score(A) \neq \score(A')$. By Theorem \ref{thm:PO}, we have $\USW(A)=\USW(A')$. By Lemma \ref{lem:sequence}, we can obtain another utilitarian optimal allocation $A''$ that is a lexicographic improvement of $A'$ by decreasing the value of the $j$-th element of $\score(A')$ by $1$ and increasing the value of the $i$-th element of $\score(A')$ by $1$, where $\score(A')_j \geq \score(A')_i+2$. Applying Lemma \ref{lem:convex}, we get $\Phi(\score(A')) >\Phi(\score(A''))$, which gives us the desired contradiction. 

The equivalence $2 \Leftrightarrow 3$ immediately holds by the fact that $-\Psi$ is a symmetric strictly convex function. 

To prove $3 \Leftrightarrow 4$, let $A$ be a leximin allocation, and let $A'$ be an MNW allocation. Again, we will show that $\score(A')$ is the same as $\score(A)$, which by the uniqueness of the leximin valuation vector and symmetry of $\NW$, proves the theorem statement. Let $N_{>0}(A)$ (respectively, $N_{>0}(A')$) be the agent subset to which we allocate bundles of positive values under leximin allocation $A$ (respectively, MNW allocation $A'$). By definition, the number $n'$ of agents who get positive values under leximin allocation $A$ is the same as that of MNW allocation $A'$. Now we denote by ${\bar \score}(A)$ (respectively, ${\bar \score}(A')$) the vector of the non-zero components $v_i(A_i)$ (respectively, $v_i(A'_i)$) arranged in non-decreasing order. Assume towards a contradiction that  ${\bar \score}(A) >_L {\bar \score}(A')$. 
Since $A'$ maximizes the product $\NW(A')$ when focusing on $N_{>0}(A')$ only, the value $\sum_{i \in N_{>0}(A')}\log v_i(A'_i)$ is maximized. 
However, $\Psi(\bfx)= \sum^{n'}_{i=1}\log x_i$ is a symmetric concave function for $\bfx \in \Z^n$ with each $x_i>0$. Thus, by a similar argument as before, one can show that $\Psi({\bar \score}(A')) > \Psi({\bar \score}(A))$, a contradiction. 
This completes the proof. 
\end{proof}

\begin{remark}[The Pigou-Dalton principle]\label{rem:PDP}
\normalfont Maximizing any symmetric strictly concave function (equivalently, minimizing any symmetric strictly convex function) of valuations among all utilitarian optimal allocations is consistent with the well-known \textit{Pigou-Dalton transfer principle} (see, e.g., \citet{moulin2004fair}) of welfare economics in a strong sense, and the proof is similar to that of Lemma~\ref{lem:sequence}. We formalize this notion in Appendix~\ref{app:PDP}.\footnote{We thank an anonymous reviewer for pointing out this connection.} \remend
\end{remark}
\appendixproof{Connections between Theorem~\ref{thm:convex} and the Pigou-Dalton principle}
{\label{app:PDP} Here, we will elaborate on Remark~\ref{rem:PDP} from Section~\ref{sec:binsub_mnwlm}. We begin with an introduction to the Pigou-Dalton principle (PDP). Consider a fair allocation instance --- where agents have arbitrary valuation functions --- which admits two allocations $A$ and $A'$ with the following properties: $v_1(A_1)+v_2(A_2)=v_1(A'_1)+v_2(A'_2)$; $v_i(A_i)=v_i(A'_i)$ $\forall i \in N \setminus \{1,2\}$; $v_1(A_1) < v_1(A’_1) < v_2(A_2)$, the first and last conditions implying $v_1(A_1) < v_2(A’_2) < v_2(A_2)$ as well. In other words, $A'$ can be obtained from $A$ by transferring some utility from the ``richer" agent $2$ to the ``poorer'' agent $1$, everything else (including the utilitarian social welfare) remaining the same and without making agent $2$ poorer under $A'$ than agent $1$ was under $A$. Then, PDP stipulates an ordering that (strictly) prefers $A'$, which reduces the inequality between agents $1$ and $2$ to $A$ (the choice of agents is w.l.o.g.). 

We know that the utilitarian social welfare function itself satisfies PDP weakly. We will now show that, if there are multiple utilitarian optimal allocations, one that is preferred to another by PDP also has a \textit{strictly} higher value of an arbitrary symmetric, strictly concave function $\Psi:\Z^n \rightarrow \R$ applied to agents' utilities. 
\begin{proposition}\label{prop:PDP}
    Let $\Psi:\Z^n \rightarrow \R$ be an arbitrary symmetric, strictly concave function applied to agents' valuations. 
    If an instance where all agents have arbitrary monotone valuation functions admits two utilitarian optimal allocations $A$ and $A'$ with respective valuation vectors $\mathbf{r}(A)$ and $\mathbf{r}(A')$ such that $v_1(A_1) < v_1(A’_1) < v_2(A_2)$ and $v_i(A_i)=v_i(A'_i)$ $\forall i \in N \setminus \{1,2\}$, then $\Psi(\mathbf{r}(A'))>\Psi(\mathbf{r}(A))$.
\end{proposition}
\begin{proof}
    Since both allocations have the same (optimal) utilian social welfare, it is obvious that 
    \begin{align}
        &v_1(A_1)+v_2(A_2)=v_1(A'_1)+v_2(A'_2) \notag\\
        \Rightarrow \quad &v_1(A'_1) - v_1(A_1) = v_2(A_2) - v_2(A'_2). \label{alpha}
    \end{align}
    By definition, PDP prefers $A'$ to $A$. Moreover,
    \begin{align}
        &v_1(A_1) < v_1(A’_1) < v_2(A_2) \notag\\
        \Rightarrow \quad &0 < v_1(A’_1) - v_1(A_1) < v_2(A_2) - v_1(A_1). \label{beta}
    \end{align}
   Let $\alpha \triangleq v_1(A'_1) - v_1(A_1) = v_2(A_2) - v_2(A'_2)$ by \eqref{alpha}; $\beta \triangleq v_2(A_2) - v_1(A_1)$. By \eqref{beta}, $0 < \frac{\alpha}{\beta} < 1$.
   
   Since the vector $\mathbf{r}(A')$ can be obtained from the vector $\mathbf{r}(A)$ by increasing the first entry by $v_1(A'_1) - v_1(A_1)$ and reducing the second by $v_2(A_2) - v_2(A'_2)$, we can write $\mathbf{r}(A')=\mathbf{r}(A)+\alpha(\chi_1-\chi_2)$ where $\chi_i$ the $n$-dimensional incidence vector whose $j$-th component $1$ if $j=i$, and $0$ otherwise, as in Section~\ref{sec:binsub_mnwlm}.
   
   Let us further define a vector $\mathbf{y} \triangleq \mathbf{r}(A)+\beta(\chi_1-\chi_2)$. It is easy to see that $\mathbf{y}$ is a permutation of $\mathbf{r}(A)$ with entries $1$ and $2$ swapped, hence $\Psi(\mathbf{r}(A))=\Psi(\mathbf{y})$ due to the symmetry of $\Psi$.
   
   Finally, let $\lambda \triangleq 1-\frac{\alpha}{\beta} \in (0,1)$. Simple algebra shows that $\lambda \mathbf{r}(A)+(1-\lambda)\mathbf{y} = \mathbf{r}(A')$. Hence, by the strict concavity of $\Psi$, we get $\Psi(\mathbf{r}(A')) > \lambda\Psi(\mathbf{r}(A))+(1-\lambda)\Psi(\mathbf{y})
       = \lambda\Psi(\mathbf{r}(A))+(1-\lambda)\Psi(\mathbf{r}(A)=\Psi(\mathbf{r}(A))$.
\end{proof}
}

The above theorem does not generalize to the non-binary case: Example \ref{ex:leximin:MNW} represents an instance where neither leximin nor MNW allocation is utilitarian optimal. 
\begin{example}\label{ex:leximin:MNW}
	Consider an instance with assignment valuations given as follows. 
	Suppose there are three groups, each of which contains a single agent, Alice, Bob, and Charlie, respectively, and three items with weights given in Table \ref{table}. 
	\begin{table}[h!]
		\begin{center}
			\upshape
			\setlength{\tabcolsep}{6.4pt}
			\scalebox{1}{
				\begin{tabular}{rccccc}
					\toprule
					&
					\multicolumn{3}{l}{\!\!\!
						\begin{tikzpicture}[scale=0.57, transform shape, every node/.style={minimum size=7mm, inner sep=1.2pt, font=\huge}]	
						\node[draw, circle](2)  at (1.2,0) {$1$};
						\node[draw, circle](3)  at (2.5,0) {$2$};
						\node[draw, circle](4)  at (4,0) {$3$};
						\end{tikzpicture}\!\!\!\!
						\vspace{-2pt}
					} \\
					\midrule
					Alice:\!\! &  2 & 1 & 0 \\
					Bob:\!\! &  2 & 1 & 0   \\
					Charlie:\!\! &  0 & 2.9 & 0.1  \\
					\bottomrule
				\end{tabular}
			}
		\end{center}
		\caption{An instance where neither leximin nor MNW allocation is utilitarian optimal.}
		\label{table}
	\end{table}
	The unique leximin and MNW allocation is the allocation that assigns Alice to the first item, Bob to the second item, and Charlie to the third item; each agent has positive utility at the allocation and the total utilitarian social welfare is $3.1$. However, the utilitarian optimal allocation assigns Alice to nothing, Bob to the first item, and Charlie to the second item, which yields the total utilitarian social welfare $4.9$.  \egend
\end{example}

Combining the above characterization with the results of Section~\ref{sec:utiloptef1}, we get the following fairness-efficiency guarantee for \bsub valuations.

\begin{corollary}\label{cor:bin_mnw}
For \bsub valuations, any clean leximin or MNW allocation is EF1. 
\end{corollary}
\begin{proof}
	Since both leximin and MNW allocations are 
	Pareto-optimal, they maximize the utilitarian social welfare,
	by Theorem \ref{thm:PO}. 
	By Theorem \ref{thm:convex} and the fact that the function $\Phi(A) \triangleq \sum_{i \in N}v_i(A_i)^2$ is a symmetric strictly convex function, any leximin or MNW allocation is a utilitarian optimal allocation that minimizes $\Phi(A)$ among all utilitarian optimal allocations; hence, if such an allocation is clean, it must be EF1 by Corollary \ref{cor:sumsquares}.
\end{proof}

\section{Assignment valuations with binary gains}\label{sec:boxs}
We now consider the special but practically important case when valuations come from maximum matchings. 
For this class of valuations, we show that invoking Theorem \ref{thm:PO}, one can find a leximin or MNW allocation in polynomial time, by a reduction to the network flow problem. 

\begin{theorem}\label{thm:leximin:poly}
For assignment valuations with binary marginal gains, one can find a leximin or MNW allocation in polynomial time. 
\end{theorem}
\begin{proof}
The problem of finding a leximin allocation under the \boxs valuation class can be reduced to that of finding an integral balanced flow (or increasingly-maximal integer-valued flow) in a network, which has been recently shown to be polynomial-time solvable \cite{Murota2019}. 

Specifically, for a network $D=(V,A)$ with source $s$, sink $t$, and a capacity function $c:A \rightarrow \Z$, a {\em balanced flow} is a maximum integral feasible flow where the out-flow vector from the source $s$ to the adjacent vertices $h$ is lexicographically maximized among all maximum integral feasible flows; that is, the smallest flow-value on the edges $(s,h)$ is as large as possible, the second smallest flow-value on the edges $(s,h)$ is as large as possible, and so on. \citet{Murota2019} show that one can find a balanced flow in strongly polynomial time (see Section $7$ in \citet{Murota2019}).

Now, given an instance of assignment valuations with binary marginal gains, we build the following instance $(V,A)$ of a network flow problem. Let $N_h$ denote the set of members in each group $h$. We first create a source $s$ and a sink $t$. We create a vertex $h$ for each group $h$, a vertex $i$ for each member $i$ of some group, and a vertex $o$ for each item $o$. We construct the edges of the network as follows:  
\begin{itemize}
\item for each group $h$, create an edge $(s,h)$ with capacity $m$; and
\item for each group $h$ and member $i$ in group $h$, create an edge $(h,i)$ with unit capacity; and
\item for each member $i$ of some group and item $o$ for which $i$ has positive weight $u_{io}$ (i.e. $u_{io}=1$), create an edge $(i,o)$ with unit capacity; and
\item for each item $o$, create an edge $(o,t)$ with unit capacity. 
\end{itemize}
\begin{figure}[htb]
	\centering
	\begin{tikzpicture}[scale=0.9, transform shape, every node/.style={minimum size=6mm, inner sep=1pt}]
	\node[draw, circle](1) at (0,0) {$s$};
	\node at (2,2.6) {Groups};
	\node at (4,2.7) {Members};
	\node at (6,2.7) {Items};
	
	\node[draw, circle](2) at (2,1) {};
	\node[draw, circle](3) at (2,0) {$h$};
	\node[draw, circle](4) at (2,-1) {};
	\draw[->, >=latex] (1)--(2);
	\node at (0.95,0.66) {$m$};
	\draw[->, >=latex] (1)--(3);
	\node at (0.95,0.14) {$m$};
	\draw[->, >=latex] (1)--(4);
	\node at (0.95,-0.31) {$m$};
	
	\node[draw, circle](5) at (4,2) {};
	\node[draw, circle](6) at (4,1) {};
	\node[draw, circle](7) at (4,0) {$i$};
	\node[draw, circle](8) at (4,-1) {};
	\node[draw, circle](9) at (4,-2) {};
	
	\draw[->, >=latex] (2)--(5); \draw[->, >=latex] (2)--(6); \draw[->, >=latex] (3)--(7); \draw[->, >=latex](4)--(8); \draw[->, >=latex] (4)--(9);
	
	
	\node[draw, circle](10) at (6,2) {};
	\node[draw, circle](11) at (6,1) {};
	\node[draw, circle](12) at (6,0) {$j$};
	\node[draw, circle](13) at (6,-1) {};
	\node[draw, circle](14) at (6,-2) {};
	
	\draw[->, >=latex] (5)--(10); \draw[->, >=latex]  (5)--(11); \draw[->, >=latex]  (5)--(13);
	\draw[->, >=latex] (6)--(11);
	\draw[->, >=latex] (7)--(11); \draw[->, >=latex]  (7)--(14);
	\draw[->, >=latex] (8)--(13);
	\draw[->, >=latex] (9)--(12); \draw[->, >=latex]  (9)--(14);
	
	\node[draw, circle](15) at (8,0) {$t$};
	\draw[->, >=latex] (10)--(15); \draw[->, >=latex]  (11)--(15); \draw[->, >=latex]  (12)--(15); \draw[->, >=latex]  (13)--(15); \draw[->, >=latex]  (14)--(15);
	\end{tikzpicture}
	\caption{An illustrative network flow instance constructed in the proof of Theorem \ref{thm:leximin:poly}: each edge is either labeled with its capacity or has unit capacity. \label{fignet}}
\end{figure}
See Figure~\ref{fignet} 
for an illustration of the network. We will show that an integral balanced flow $f:A \rightarrow \Z$ of the constructed network corresponds to a leximin allocation. Consider an allocation $A^f$ where each group receives the items $o$ for which some member $i$ of the group has positive flow $f(i,o)>0$. 

It is easy to see that the allocation $A^f$ maximizes the utilitarian social welfare since the flow $f$ is a maximum integral feasible flow. Thus, by Theorem \ref{thm:PO}, $A^f$ has the same utilitarian social welfare as any leximin allocation. To see balancedness, observe that the amount of flow from the source $s$ to each group $h$ is the valuation of $h$ for bundle $A^f_h$, i.e., $f(s,h)=\sum_{i \in N_h}f(h,i)=v_h(A^f_h)$. Indeed if $v_h(A^f_h)>f(s,h)$, then it would contradict the optimality of the flow $f$; and if $v_h(A^f_h)<f(s,h)$, it would contradict the fact that $v_h(A^f_h)$ is the value of a maximum-size matching between $A^f_h$ and $N_h$. Thus, among all utilitarian optimal allocations, $A^f$ lexicographically maximizes the valuation of each group, and hence $A_f$ is a leximin allocation. By Theorem \ref{thm:convex}, the leximin allocation $A_f$ is also MNW.
\end{proof}

In contrast with assignment valuations with binary marginal gains, the problem of computing a leximin or MNW allocation becomes intractable for weighted assignment valuations even when there are only two agents, the following theorem shows (the proof is in Appendix~\ref{app:proof_hardness:leximin}). 

\begin{theorem}\label{hardness:leximin}
For two agents with general assignment valuations, it is NP-hard to compute a leximin or MNW allocation.
\end{theorem}
\appendixproof{Proof of Theorem \ref{hardness:leximin}}{\label{app:proof_hardness:leximin}
	\begin{stmnt*}
		For two agents with general assignment valuations, it is NP-hard to compute a leximin or MNW allocation.
	\end{stmnt*}
\begin{proof}
The reduction is similar to the hardness reduction for two agents with identical additive valuations \cite{Nguyen2013,Ramezani2010}. 
We give a Turing reduction from {\sc Partition}. Recall that an instance of {\sc Partition} is given by a set of positive integers $W=\{w_1,w_2,\ldots,w_m\}$; it is a `yes'-instance if and only if it can be partitioned into two subsets $S_1$ and $S_2$ of $W$ such that the sum of the numbers in $S_1$ equals the sum of the numbers in $S_2$.

Consider an instance of {\sc Partition} $W=\{w_1,w_2,\ldots,w_m\}$. We create $m$ items $1,2, \ldots,m$, two groups $1$ and $2$, and $m$ individuals for each group where every individual has a weight $w_j$ for item $j$. Observe that for each group, the value of each bundle $X$ is the sum $\sum_{w_j \in X}w_j$: the number of members in the group exceeds the number of items in $X$, and thus one can fully assign each item to each member of the group. 

Suppose we had an algorithm which finds a leximin allocation. Run the algorithm on the allocation problem constructed above to obtain a leximin allocation $A$. It can be easily verified that the instance of {\sc Partition} has a solution if and only if $v_1(A_1)=v_2(A_2)$. 
Similarly, suppose we had an algorithm which finds an MNW allocation, and run the algorithm to find an MNW allocation $A'$. Since the valuations are identical, the utilitarian social welfare of the MNW allocation is the sum $\sum_{w_j \in W}w_j$, which means that the product of the valuations is maximized when both groups have the same realized valuation. Thus, the instance of {\sc Partition} has a solution if and only if $v_1(A'_1)=v_2(A'_2)$.
\end{proof}
}

\section{Discussion}\label{sec:disc}
We studied allocations of indivisible goods under submodular valuations with binary marginal gains in terms of the interplay among envy, efficiency, and various welfare concepts. We showed that three seemingly disjoint outcomes --- minimizers of arbitrary symmetric strictly convex functions among utilitarian optimal allocations, the leximin allocation, and the MNW allocation --- coincide in this class of valuations. 
Since the class of matroid rank functions is rather broad, our results can be applied to settings where agents' valuations are induced by a matroid structure. Beyond the domains described in this work, these include several others. For example, {\em partition matroids} model instances where agents' have access to different item types, but can only hold a limited number of each type (their utility is the total number of items they hold); a variety of other domains, such as spanning trees, independent sets of vectors, coverage problems and more admit a matroid structure (see \citet{oxley2011matroid} for an overview). Indeed, a well-known result in combinatorial optimization states that {\em any} agent valuation structure where the greedy algorithm can be used to find the (weighted) optimal bundle, is induced by some matroid \cite[Theorem 1.8.5]{oxley2011matroid}.

We will conclude with additional implications of this work, some work in progress, and directions for further research.

\paragraph{Complete vs clean EF1, PO allocations} In Section~\ref{sec:utiloptef1}, we showed that focusing on \textit{clean} allocations helps us design an elegantly simple and polynomial-time algorithm (Algorithm~\ref{alg_bin}) that computes an EF1 and utilitarian optimal (but potentially \textit{incomplete}) allocation under \bsub valuations --- in other words, cleaning followed by further processing of a utilitarian optimal allocation is \emph{sufficient} for achieving the EF1+PO combination. It is still an interesting open problem whether cleaning (and hence withholding some items) is \textit{necessary} for achieving the desired fairness-efficiency combination. Intuitively, a complete, utilitarian optimal allocation may induce \emph{avoidable} envy levels among agents\footnote{The observation that retaining items with zero marginal gain in an agent's bundle can make other agents envious is also a barrier to extending our results to a setting without \textit{free disposal}, i.e. a setting where an agent incurs a cost for having an item removed from her bundle.} --- agent $j$'s bundle $A_j$ might include a subset $S$ that contributes nothing to the overall welfare, i.e. $v_j(A_j\setminus S)=v_j(A_j)$, but makes another agent $i$ envious of $j$ up to more than $1$ items, i.e. $v_i(A_i) < v_i(A_j \setminus \{o\})$ for every $o \in A_j$ but $v_i(A_i) \ge v_i(A_j \setminus (S \cup \{o'\}))$ for some $o' \in A_j \setminus S$. However, an instance that admits no complete, EF1, utilitarian optimal allocation (thus proving the necessity of cleaning in general) remains elusive. Please refer to Appendix~\ref{proof_prop:clean_size} for further comments on the interplay between cleanness and completeness.

\paragraph{More general valuation functions} 
An imperative line of future work is investigating which of our findings extend to more general valuation functions. There are several known extensions to matroid structures, with deep connections to submodular optimization \cite[Chapter 11]{oxley2011matroid}. Matroid rank functions are submodular functions with binary marginal gains; however, general submodular functions (i.e. those with non-negative real marginal gains) admit some matroid structure which may potentially be used to extend our results to more general settings.

An obvious generalization of the \bsub valuation function class is the class of submodular valuation functions with \emph{subjective} binary marginal gains, i.e. $\Delta_i(S;o) \in \{0,\lambda_i\}$ for some agent-specific constant $\lambda_i > 0$, for every $i \in N$. For this valuations class that we call $(0,\lambda_i)$-\textsc{SUB}, we can show that any clean, MNW allocation is still EF1 (clean bundles being defined the same way as for \bsub valuations) but the leximin and MNW allocations no longer coincide and leximin no longer implies EF1 (the details are in Appendix~\ref{sec:subjbin}). 

For general assignment valuations (i.e. members have positive real weights for items), we have no theoretical guarantees yet. However, we ran experiments on a real-world data set, 
comparing the performance of a heuristic extension of Algorithm~\ref{alg_bin} (Section~\ref{sec:utiloptef1}) to real-valued individual-item utilities (weights) with \citet{lipton2004approximately}'s envy graph algorithm in terms of the number of items \emph{wasted} (left unassigned or assigned to individuals with zero utility for it although another agent has positive utility for the item). These experiments, described in detail in Appendix~\ref{sec:assval}, suggest that approximate envy-freeness can often be achieved in practice simultaneously with good efficiency guarantees even for this larger valuation class.


It is important to note that the class of rank functions of matroids is a subclass of the well-known \emph{gross substitutes} (GS) valuations \cite{gul1999grosseq,kelso1982grosssubs}. 
A promising research direction is to investigate PO+EF1 existence for GS valuations.

\paragraph{Other fairness criteria} The fairness concept we consider here is (approximate) envy-freeness. An obvious next step is to explore other criteria such as \emph{proportionality} (each agent gets at least $\nicefrac{1}{n}$ of her valuation of the full collection of goods $O$), the \emph{maximin share guarantee} or MMS (each agent gets at least as much value as she would realize if allowed to partition $O$ completely among all agents knowing that she would receive her least favorite part), \emph{equitability} (all agents have equal realized valuations), etc. (see, e.g. \citet{caragiannis2016unreasonable,freeman2019equitable} and references therein for further details) 
for \bsub valuations. We present our results from a preliminary exploration of these questions in Appendix~\ref{sec:other_fairness}. 
It is worthwhile to summarize here one of these results that extends a recent paper by \citet{freeman2019equitable}.
This paper shows that an allocation that is equitable up to one item or EQ1 (a relaxation of equitability in the same spirit as EF1) and PO may not exist even for binary additive valuations; however, for this valuation class,  it can be verified in polynomial time whether an EQ1, EF1 and PO allocation exists and, whenever it does exist, it can also be computed in polynomial time (for the time complexity result, they show that such an allocation is MNW). We can generalize this result to \boxs valuations: we first show that any EQ1 and PO allocation under the \bsub valuation class, if it exists, is leximin, then invoke Corollary~\ref{cor:bin_mnw} to conclude that it must be EF1, and finally Theorem~\ref{thm:leximin:poly} to establish its polynomial-time complexity for the \boxs class (the full proof is in Appendix~\ref{sec:approx_eq}). 

 \paragraph{Implications for diversity} Finally, the analysis of submodular valuations ties in with existing works on diversity in various fields from biology to machine learning (see, e.g. \citet{jost2006entropy,celis2016fair,celis2018fair}). A popular measurement for how diverse a solution is is to apply one of several concave functions called \emph{diversity indices} to the proportions of the different entities/attributes (with respect to which we wish to be diverse) in the solution, e.g. the Shannon entropy and the Gini-Simpson index: if we denote the maximum $\USW$ of one of the problem instances studied in this paper by $U^*$ and agent $i$'s realized valuation in a utilitarian optimal allocation as $u_i$, then the above two indices can be expressed as $-\sum_{i \in N} (\nicefrac{u_i}{U^*}) \ln (\nicefrac{u_i}{U^*})$ and $1-\sum_{i \in N} (\nicefrac{u_i}{U^*})^2$ respectively such that $\sum_{i \in N} u_i =U^*$. Thus, Theorem~\ref{thm:convex} also shows that, for \bsub valuations, the MNW or leximin principle maximizes among all utilitarian optimal allocations commonly used \emph{diversity indices} applied to shares of the agents in the optimal $\USW$. It will be interesting to explore potential connections of this interpretation to recent work on {\em soft} diversity framed as convex function optimization \cite{ahmed2017diverse}.
 
 


\begin{acks}
	Benabbou was supported by the ANR project 14-CE24-0007-01-Cocorico-CoDec, Chakraborty and Zick by a Singapore MOE grant (no. R-252-000-625-133) and a Singapore NRF Research Fellowship (no. R-252-000-750-733), and Ayumi Igarashi by the KAKENHI Grant-in-Aid for JSPS Fellows
	no. 18J00997 and JST, ACT-X. Most of this work was done when Chakraborty and Zick were employed at the National University of Singapore (NUS), and Igarashi was a research visitor at NUS, supported by Singapore MOE Grant R-252-000-625-133.
	The authors would like to thank Edith Elkind, Warut Suksompong, Dominik Peters, and Tushant Jha for valuable insights and feedback on earlier versions of the paper. Thanks are also due to the anonymous reviewers of GAIW 2020, the Harvard CRCS Workshop on AI for Social Good 2020, and SAGT 2020 for their feedback on earlier versions of this paper.
\end{acks}
\bibliographystyle{ACM-Reference-Format}  
\bibliography{abb,TEF1refs}  

\clearpage
\appendix

\section*{Appendices}

\section{Omitted proofs, examples, and remarks}\label{app:egs}
\appendixProofText

\section{Submodularity with subjective binary gains}\label{sec:subjbin}
An obvious generalization of the \bsub valuation function class is the class of submodular valuation functions with \emph{subjective} binary marginal gains: agent $i$'s bundle-valuation function $v_i(\cdot)$ is said to have subjective binary marginal gains if $\Delta_i(S;o) \in \{0,\lambda_i\}$ for some agent-specific constant $\lambda_i > 0$, for every $i \in N$. We define clean bundles and clean allocations for this function class exactly as we did for \bsub valuations in Section~\ref{sec:prelims}.

Understandably, most of the properties of allocations under \bsub valuations do not extend to this more general setting. It is obvious that Pareto optimality does not imply utilitarian optimality (e.g. consider an instance with two agents and one item which the agents value at $1$ and $2$ respectively: assigning the item to agent $1$ is PO but not utilitarian optimal). Moreover, the leximin allocation may not be EF1, as shown by the following example where both agents have additive valuations.
\begin{example}\label{ex:lmin_nef1}
	Suppose $N=[2]$; $O=\{o_1,o_2,o_3,o_4\}$; the valuations are additive with $v_1(\{o_4\})=0$, $v_1(\{o\})=1$ $\forall o \in O \setminus \{o_1\}$, and $v_2(\{o\})=3$ $\forall o \in O$. It is straightforward to check that the unique leximin allocation is $A_1=\{o_1,o_2,o_3\}$, $A_2=\{o_4\}$. Under this allocation, 
	$v_2(A_1\setminus\{o\})=6>3=v_2(A_2)$ for every $o \in A_1$ --- in fact, at least two (any two) items must be removed from $A_1$ for agent $2$ to stop envying agent $1$. \egend
\end{example}
Note another difference of this valuation class from the class of \bsub valuations that is also evidenced by Example~\ref{ex:lmin_nef1}: the leximin and MNW allocations may not coincide. In this example, any allocation $A$ that gives two of the items $\{o_2,o_3,o_4\}$ to agent $1$ and the rest to agent $2$ is MNW, with $v_1(A_1)=2$ and $v_2(A_2)=6$, so that $\NW(A)=12$; such an allocation is also EF1 (in fact, envy-free) since $v_1(A_2)=1<2=v_1(A_1)$ and $v_2(A_1)=6=v_2(A_2)$. This is not an accident, as the following theorem shows.
\begin{theorem}\label{thm:subjbin_mnw}
	For agents having submodular valuation functions with subjective binary marginal gains, any clean, MNW allocation is EF1.
\end{theorem}
Since our valuation functions are still submodular, the transferability property (Lemma~\ref{lem:revoc_realloc}) still holds. Two other components of the proof of Theorem~\ref{thm:subjbin_mnw} are natural extensions of Propositions~\ref{prop:clean_size} and Lemma~\ref{lem:envy_size} --- Proposition~\ref{prop:subjbin_clean_size} and Lemma~\ref{lem:subjbin_envy_size} below, respectively:
\begin{proposition}\label{prop:subjbin_clean_size}
	For submodular valuations with subjective binary marginal gains defined by agent-specific positive constants $\lambda_i$ $\forall i \in N$, $A$ is a clean allocation if and only if $v_i(A_i)=\lambda_i|A_i|$ for each $i \in N$.
\end{proposition}
\begin{proof}
	Consider an arbitrary bundle $S \subseteq O$ such that $S=\{o_1,o_2, \dots, o_r\}$ for some $r \in [m]$ w.l.o.g. Let $S_0=\emptyset$ and $S_t=S_{t-1} \cup \{o_t\}$ for every $t \in [r]$. Then, an arbitrary agent $i$'s valuation of bundle $S$ under marginal gains in $\{0,\lambda_i\}$ is 
	\begin{align}
	 v_i(S)=\sum_{t = 1}^r \Delta_i (S_{t-1};o_t) \le \sum_{t = 1}^r \lambda_i = \lambda_i r = \lambda_i |S|.\label{ineq_val_size}
	 \end{align}
	Now, if agent $i$'s allocated bundle under an allocation $A$ has a valuation $v_i(A_i) = \lambda_i|A_i|$, then her marginal gain for any item in $o \in A_i$ is given by
	\begin{align*}
	v_i(A_i) - v_i(A_i \setminus \{o\})&=\lambda_i|A_i|-v_i(A_i \setminus \{o\})\\
	&\ge \lambda_i|A_i| - \lambda_i(|A_i|-1)\\
	&=\lambda_i
	>0,
	\end{align*} 
	where the first inequality follows from Inequality~\eqref{ineq_val_size} and the fact that $|A_i \setminus \{o\}|=|A_i|-1$. This means that the bundle $A_i$ is clean and, since this holds for every $i$, the allocation is clean. This completes the proof of the ``if" part.
	
	If allocation $A$ is clean, then we must have $\Delta_i(A_i\setminus \{o\}; o) > 0$ for every $o \in A_i$ for every $i \in N$. Let us define an arbitrary agent $i$'s bundle $A_i$  as $S$ above, so that $|A_i|=r$. Then, since $S_{t-1} \subseteq A_i\setminus \{o_t\}$ for every $t \in [r]$, submodularity dictates that 
	\begin{align*}
	\Delta_i (S_{t-1};o_t)
	\ge \Delta_i (A_i\setminus \{o_t\};o_t)>0 \quad \forall t \in [r].
	\end{align*} 
	Since $\Delta_i (S_{t-1};o_t) \in \{0,\lambda_i\}$ with $\lambda_i > 0$, the above inequality implies that $\Delta_i (S_{t-1};o_t) = \lambda_i$ $\forall t \in [r]$. Hence,
	\[v_i(A_i)=\sum_{t = 1}^r \Delta_i (S_{t-1};o_t) = \sum_{t = 1}^r \lambda_i = \lambda_i r = \lambda_i |A_i|.\]
	This completes the proof of the ``only if" part.
\end{proof}
\begin{lemma}\label{lem:subjbin_envy_size}
	For submodular functions with subjective binary marginal gains, if agent $i$ envies agent $j$ up to more than $1$ item under clean allocation $A$, then $|A_{j}| \ge |A_i|+2$.
\end{lemma}
\begin{proof}
	Since $i$ envies $j$ under $A$ up to more than $1$ item, we must have $A_j \neq \emptyset$ and  $v_i(A_i) < v_i(A_{j} \setminus \{o\})$ for every $o \in A_j$. Consider one such $o$.
	From Inequality~\eqref{ineq_val_size} in the proof of Proposition~\ref{prop:subjbin_clean_size}, $v_i(A_j \setminus \{o\}) \leq \lambda_i|A_j \setminus \{o\}|= \lambda_i(|A_j| - 1)$.  
	Since $A$ is clean, $v_i(A_i)=\lambda_i|A_i|$. Combining these, we get 
	\[
	\lambda_i|A_i|=v_i(A_i) < v_i(A_j \setminus \{o\}) \leq \lambda_i(|A_j| -1).
	\] 
	Since $\lambda_i > 0$, we have $|A_i| < |A_j| -1$, i.e. $|A_i| \le |A_j| -2$ because $|A_i|$ and $|A_j|$ are integers. 
\end{proof}
We are now ready to prove Theorem~\ref{thm:subjbin_mnw}.
\begin{proof}[Proof of Theorem~\ref{thm:subjbin_mnw}]
		Our proof non-trivially extends that of Theorem 3.2 of \citet{caragiannis2016unreasonable}. We will first address the case when it is possible to allocate items in such a way that each agent has a positive realized valuation for its bundle, i.e. $N_{\max} = N$ in the definition of an MNW allocation, and then tackle the scenario $N_{\max} \subsetneq N$. 
		
		Consider a pair of agents $1,2 \in N$ w.l.o.g. such that $1$ envies $2$ up to two or more items, if possible, under an MNW allocation $A$. Since every agent has a positive realized valuation under $A$, we have $v_i(A_i)=\lambda_i |A_i| > 0$, i.e. $|A_i| > 0$ for each $i \in \{1,2\}$. From Lemma~\ref{lem:revoc_realloc}, we know that there is an item in $A_2$ for which agent $1$ has positive marginal utility -- consider any one such item $o \in A_2$. Thus, $\Delta_1(A_1;o) > 0$, i.e. $\Delta_1(A_1;o) = \lambda_1$; also, since $A_2$ is a clean bundle, $\Delta_2(A_2 \setminus \{o\};o) > 0$, i.e. $\Delta_2(A_2 \setminus \{o\};o) = \lambda_2$.
		
		Let us convert $A$ to a new allocation $A'$ by only transferring this item $o$ from agent $2$ to agent $1$.  Hence, $v_1(A'_1)=v_1(A_1)+\Delta_1(A_1;o)=v_1(A_1)+\lambda_1$, $v_2(A'_2)=v_2(A_2)-\Delta_2(A_2 \setminus \{o\};o)=v_2(A_2)-\lambda_2$, $v_i(A'_i) = v_i(A_i)$ for each $i \in N \setminus \{1,2\}$. $\NW(A)$ is positive since $A$ is MNW and $N_{\max} = N$. Hence,
		\begin{align}
		\frac{\NW(A')}{\NW(A)}
		&= \left[\frac{v_1(A_1)+\lambda_1}{v_1(A_1)}\right]\left[\frac{v_2(A_2)-\lambda_2}{v_2(A_2)}\right]\notag \\
		&= \left[1+\frac{\lambda_1}{v_1(A_1)}\right]\left[1-\frac{\lambda_2}{v_2(A_2))}\right]\notag \\
		&=\left[1+\frac{\lambda_1}{\lambda_1|A_1|}\right]\left[1-\frac{\lambda_2}{\lambda_2 |A_2|}\right]\notag \\
		&=\left[1+\frac{1}{|A_1|}\right]\left[1-\frac{1}{ |A_2|}\right]\notag \\
		&=1+\frac{|A_2| - |A_1| - 1}{|A_1||A_2|}, \notag \\
		&\ge 1+\frac{(|A_1|+2) - |A_1| - 1}{|A_1||A_2|}, \notag \\
		&\ge 1+\frac{1}{|A_1||A_2|},  \notag \\
		&> 1. \notag
		\end{align}
		Here, the third equality comes from  Proposition~\ref{prop:subjbin_clean_size} since $A$ is clean, and the first inequality from Lemma~\ref{lem:subjbin_envy_size} due to our assumption.
		But $\NW(A') > \NW(A)$ contradicts the optimality of $A$, implying that any agent can envy another up to at most $1$ item under $A$. 
		
		 This completes the proof for the $N_{\max}=N$ case. The rest of the proof mirrors the corresponding part of the proof of \citet{caragiannis2016unreasonable}'s Theorem 3.2. If $N_{\max} \subsetneq N$, it is easy to see that there can be no envy towards any $i \not\in N_{\max}$: this is because we must have $v_i(A_i)=0$ for any such $i$ from the definition of $N_{\max}$, which in turn implies that $A_i = \emptyset$ since $A$ is clean; hence, $v_j(A_i)=0$ for every $j \in N$. Also, for any $i,j \in N_{\max}$, we can show exactly as in the proof for the $N_{\max}=N$ case above that there cannot be envy up to more than one item between them, since $A$ maximizes the Nash welfare over this subset of agents $N_{\max}$. Suppose for contradiction that an agent $i \in N \backslash N_{\max}$ envies some $j \in N_{\max}$ up to more than one item under $A$. Then, from Lemma~\ref{lem:revoc_realloc}, there is one item $o_1 \in A_j$ w.l.o.g. such that $v_i(\{o_1\})=\Delta_i(\emptyset;o_1)=\Delta_i(A_i;o_1)>0$. Moreover, since $A$ is clean, 
		 \begin{align*}
		 v_j(A_j \setminus \{o_1\})&=v_j(A_j) - \Delta_j (A_j \setminus \{o_1\};o_1)\\
		 &= \lambda_j |A_j| -\lambda_j\\
		 &= \lambda_j (|A_j|-1)\\
		 &\ge \lambda_j (|A_i|+1)\\
		 &= \lambda_j > 0,
		 \end{align*} 
		 where the first inequality comes from Lemma~\ref{lem:subjbin_envy_size}. Thus, if we  transfer $o_1$ from $j$ to $i$ and leave all other bundles unchanged, then every agent in $N_{\max} \cup \{i\}$ will have a positive valuation under the new allocation. This contradicts the maximality of $N_{\max}$. Hence, any $i \in N\backslash N_{\max}$ must be envy-free up to one item towards any $j \in N_{\max}$.
\end{proof}

\section{General assignment valuations}\label{sec:assval}
In this section, we address the fair and efficient allocation of items to agents who have general assignment or OXS valuations, as defined in Section~\ref{subsec:defs}. Recall that an agent with such a valuation function is equivalent to a group with multiple members each having an arbitrary non-negative weight for each item. As such, we will henceforth use the terms ``group" and ``agent" interchangeably.

We know that, for arbitrary non-negative monotone valuations, the classic \textit{envy graph algorithm} due to \citet{lipton2004approximately} produces a complete, EF1 allocation that does not, however, come with any efficiency guarantee (except completeness, of course). The trick is to iterate over the items and allocate each to an agent that is currently not envied by any other agent (the existence of such an unenvied agent can be guaranteed by \textit{de-cycling}, if necessary, the graph induced by a directed edge from every envious agent to every agent that it envies: see \citet{lipton2004approximately} for details).

 \citet{benabbou2019fairness} focus on fair allocation to \textit{types} that are, in fact, agents/groups with OXS valuations; they use a natural extension of this procedure that they denote by Algorithm \textbf{H}. In an iteration of Algorithm \textbf{H}, we do not give an arbitrary unallocated item to an arbitrary unenvied agent; instead, we find an item-agent pair having the maximum marginal utility among all currently unenvied agents and all unallocated items (breaking further ties uniformly at random, say), and allocate that item to that agent. Although this modification should, intuitively, improve efficiency, \citet{benabbou2019fairness} provide no formal guarantee in this regard; they evaluate the performance of Algorithm \textbf{H} in experiments where all agents have OXS valuations in terms of \textit{waste} which they define as follows: under a complete allocation $A$, an item $o$ is said to be \textit{wasted} if it has positive marginal utility for some group $h$ under $A$ (i.e. $v_h(A_h \cup \{o\})>v_h(A_h)$) but is allocated to another group $h'$ (i.e. $o \in A_{h'}$) where it is either unassigned or assigned to a member $i \in N_{h'}$ with zero weight for it (i.e. $u_{i,o}=0$), under the particular optimal matching of $A_{h'}$ to $N_{h'}$. The waste of a run of Algorithm \textbf{H} is defined as the percentage of the total number of items that are wasted under the complete allocation produced by Algorithm \textbf{H}.

Here, we ask whether the concept of \textit{envy-induced transfers} ($\EIT$) presented in Algorithm~\ref{alg_bin} for matroid rank valuations (Section~\ref{sec:utiloptef1}) can be used to compute fair and efficient allocations (perhaps in some approximate sense) under more general monotone submodular valuation functions. This is motivated in part by the fact that the transferability property (Lemma~\ref{lem:revoc_realloc}), on which the $\EIT$ concept relies, characterizes any monotone submodular function and not just matroid rank valuations. In Algorithm \ref{alg_gen}, we delineate our work in progress in this vein: a heuristic scheme that extends Algorithm~\ref{alg_bin} to general OXS valuations. 

\begin{algorithm} \small
	\DontPrintSemicolon
	\caption{\small Envy-Induced Transfers for general OXS valuations\label{alg_gen}}
	Compute a clean, utilitarian optimal allocation.\\						
	\textbf{/*Envy-Induced Transfers ($\EIT$)*/}\\
	\While{$\exists i,j \in N$ such that $i$ envies $j$ up to more than $1$ item}
	{
		Pick $i$, $j$, $o$ maximizing $\Delta_i(A_i;o) + \Delta(A_j\setminus \{o\};o)$ over all $i,j \in N$ and all $o \in O$ such that $i$ envies $j$ more than $1$ item and $\Delta_i(A_i;o)>0$.\\ 		
		$A_j \leftarrow A_j \backslash \{o\}$; $A_i \leftarrow A_i \cup \{o\}$.\\
		\If{$\exists o \in A_0$ such that $\Delta_j(A_j;o)>0$}
		{Pick $o \in A_0$ that maximizes $\Delta_j(A_j;o)$.\\ $A_j \leftarrow A_j \cup \{o\}$.}
		\If{$\exists o^* \in A_i$ that is unused}
		{$A_i \leftarrow A_i \backslash \{o^*\}$; revoked = \textbf{true}.\\
			\While{revoked \!\! = \!\! \textbf{true} \textbf{and} $\exists k$ s.t. $\Delta(A_k;o^*)\! >\! 0$}{
				Allocate $o^*$ to agent $k$ maximizing $\Delta(A_k;o^*)$.\\
				\lIf{$\exists o \in A_k$ that is unused}{\\
					\hspace{0.5cm}$A_k \leftarrow A_k \backslash \{o\}$; $o^* \leftarrow o$.
				} 
				\lElse{revoked = \textbf{false}.}
			}
			\lIf{revoked = \textbf{true}}{$A_0 \leftarrow A_0 \cup \{o^*\}$.}
		}
	}
\end{algorithm}

Algorithm \ref{alg_gen} retains the general principle of starting with a(n arbitrary) clean, utilitarian optimal allocation\footnote{Maximizing the utilitarian social welfare is NP-hard when agents have general monotone submodular valuations but can be accomplished in polynomial time under the subclass gross substitutes valuations, assuming oracle access to each valuation function \cite{lehmann2006combinatorial}. In particular, under OXS valuations (assuming that such a valuation function is specified in terms of the weights of each member of the group for all items), computing a utilitarian optimal allocation reduces to the polynomial-time solvable \textit{assignment problem} or maximum sum-of-weights matching on a bipartite graph \cite{munkres1957algo}; the result is automatically clean if we make sure that no item is assigned to an individual with zero weight for it.} and iteratively eliminating envy by transferring an item from an envied bundle to an envious agent. For matroid rank valuations, the ``donor" and the recipient of the transferred item have their valuations decreased and increased respectively by exactly $1$ for any envy-induced transfer; this is no longer the case when we remove the binary marginal utilities restriction. Hence, such a transfer does not, in general, keep the utilitarian social welfare unchanged; 
the welfare is only constrained to never exceed its starting (optimal) value computed in line 1. As an approach to minimizing the loss in welfare/efficiency due to such transfers, we employ various heuristics in Algorithm~\ref{alg_gen}:
\begin{itemize}
	\item First, in each $\EIT$ step, we transfer the item that induces the minimal decrease in --- or, equivalently, the maximal increase --- in the welfare (see lines 3-6). 
	\item Next, as the donor agent loses one of its items, it may develop a positive marginal utility for a currently withheld item; in that case, the item in $A_0$ for which it has maximal marginal utility is given to it (see lines 6-8).
	\item Finally, if an agent (group) $i$ acquires a new item $o$ due to an envy-induced transfer, at most one of its previous items, say $o^*$, may become \textit{unused}, i.e. it is no longer assigned to a member of the group under the new matching. This happens, for example, if $i$'s positive marginal utility for $o$ with its previous bundle $A_i$ was due to the fact that the member who was assigned item $o^*$ has a higher weight for $o$ than for $o^*$ and no other member prefers $o^*$ to its assigned item. In such a case, item $o^*$ is revoked from agent $i$ and allocated to the agent with maximal and strictly positive marginal utility for it (see lines 10-13). If this creates another unused item, we repeat the process until there are no unused items or the unused item has zero marginal utility for all agents -- in the latter case, the unused item is added to the withheld set (see lines 14-18).
\end{itemize}

We do not yet have theoretical guarantees for Algorithm~\ref{alg_gen}; but, if the $\EIT$ subroutine terminates, then the final allocation is EF1 and has zero waste (as defined above) by construction. To estimate the efficiency properties of our scheme, we ran numerical tests with it on a set of fair allocation instances based on a real-world data set.

In our experiments, we measure and compare the performances of Algorithm~\ref{alg_gen} and the procedure $\textbf{H}$ as described above in terms of waste (as defined above) as well as the \textit{price of fairness} (PoF) which we formally define as follows
$$\mathrm{PoF}(P) = \frac{\max \{\USW(A) \ | \ A \mbox{ is an allocation}\}}{\USW(A(P))}$$
where $A(P)$ is the allocation returned by a given procedure $P$ (Algorithm~\ref{alg_gen}  or Algorithm $\textbf{H}$) on a problem instance. Obviously, the PoF is bounded below by $1$ for any instance and lower values are better.
 
The data set we use is \textit{MovieLens-ml-1m} \cite{movieLens} which contains approximately 1,000,000 ratings (from $0$ to $5$) of 4,000 movies made by 6,000 users. To generate an instance of our allocation problem, we select $200$ movies uniformly at random ($|O|=200$) and then we only consider the users that rated at least one of these movies. Each such sample of $200$ movies defines one run of our experiments. The users are our group-members and the movies  our items. We generate agents/groups by partitioning users based on a demographic attribute; in fact, we use two attributes recorded in the data set, giving us two sets of allocation problem instances for each run:
\begin{itemize}
	\item Gender: $2$ agents (male or female, as recorded in the data set);
	\item Age: $7$ agents representing the $7$ age-groups recognized in the data set.
\end{itemize}
Moreover, for each such set (with $2$ and $7$ agents respectively), we adopt two models for the member-item weights or, equivalently, agents' valuation functions (raw and normalized ratings), giving us a $2 \times 2$ experimental design:
\begin{itemize}
	\item \textsc{Ratings}: $$v_h(S) = \max \Bigg\{\, \sum\limits_{u \in N_h} r_{u,\pi(u)} \mid \pi \in \Pi(N_h,S) \, \Bigg\}$$ where $r_{uo}$ is the user $u$'s rating of movie $o$;
	\item \textsc{Norm}: $v'_h(S)= v_h(S)/v_h(O)$,
\end{itemize}
for every agent $h$ (group $N_h$) and for any bundle of movies $S \subseteq O$. We provide the results, averaged over 50 runs, in Table \ref{tabTests}.

\begin{table}[h]
	\centering
	\begin{tabular}{ c | c | c c | c c }
		& & \multicolumn{2}{c|}{$\textbf{H}$ \cite{lipton2004approximately,benabbou2019fairness}} &   \multicolumn{2}{c}{Algorithm~\ref{alg_gen}}  \\
		&  Attribute ($\#$groups) &  \textsc{ Ratings } & \textsc{ Norm } & \textsc{ Ratings } & \textsc{ Norm } \\ \hline
		PoF & \multirow{2}{*}{Age ($7$)} &  1.01 & 1.15  & 1.05 & 1.19 \\ 
		Waste&  &  $1.25\%$ & $0.20\%$ & $0\%$ & $0\%$  \\ \hline
		PoF  & \multirow{2}{*}{Gender ($2$)} &  1.00 & 1.02  & 1.00 & 1.03 \\ 
		Waste  &  & $0.00\%$ & $0.00\%$ & $0\%$ & $0\%$ \\
	\end{tabular} 
	\caption{Experimental assessment of allocation procedures under OXS valuations. \label{tabTests}}
\end{table}

We observe that Algorithm $\textbf{H}$ has no guarantees on waste but, in practice, has negligible waste; and the waste appears to be lower for a lower number of agents in our experiments. In comparison, Algorithm~\ref{alg_gen} (which always terminated on its own for all our instances) is waste-free by design but has at least as much average PoF as Algorithm $\textbf{H}$ in all our experiments.

\section{Other fairness criteria under \bsub valuations}\label{sec:other_fairness}
In the main paper, we have focused on Pareto optimal and EF1 allocations for the \bsub valuation class. However, many other concepts have been defined and studied in the literature that formalize different intuitive ideas for what it means for an allocation of indivisible goods to be fair. In this appendix, we will investigate the implications of our results from the main paper for some alternative fairness notions.   

\subsection{Marginal envy-freeness up to one item}\label{sec:mef1}
\citet{caragiannis2016unreasonable} define an allocation $A$ to be \emph{marginally envy-free up to one item} or MEF1 if for every pair of agents $i,j\in N$ such that $i$ envies $j$, there is an item $o \in A_j$ such that
\[v_i(A_i) \ge v_i(A_i \cup A_j \setminus \{o\}) - v_i(A_i).\]
Clearly, MEF1 is in general a relaxation of EF1: it coincides with EF1 for additive valuations and is implied by EF1 for submodular valuations, as Proposition~\ref{prop:ef1_mef1} below shows. Thus, our main results trivially prove the existence (and computational tractability) of MEF1 allocations with optimal welfare guarantees (hence Pareto optimality) for \bsub valuations. However, \citet{caragiannis2016unreasonable} already established the more general existence result that, for submodular valuations, every MNW allocation is MEF1.
\begin{proposition}\label{prop:ef1_mef1}
	For submodular valuation functions, an EF1 allocation is always MEF1 but the converse is not true, even if the functions have binary marginal gains.
\end{proposition}
\begin{proof}
	Consider an EF1 allocation $A$, and pick an arbitrary pair of envious and envied agents $i, j \in N$. Let $o \in A_j$ be an item for which $v_i(A_i) \ge v_i(A_j \setminus \{o\})$. Then, since $A_i$ and $A_j$ are disjoint, submodularity implies
	\begin{align*}
	v_i(A_i \cup A_j \backslash \{o\}) -  v_i(A_i) &\le \left[v_i(A_i) + v_i(A_j \backslash \{o\})\right] -  v_i(A_i)\\
	& =  v_i(A_j \backslash \{o\})\\
	&\le v_i(A_i).
	\end{align*}
	Thus, $A$ satisfies the definition of an MEF1 allocation.
	
	The following is a counterexample to the converse. Suppose, $N=[2]$, $O=\{o_1,o_2,o_3,o_4\}$, and the valuation functions of the agents are
	\begin{align*}
	&v_1(S)=\begin{cases}
	|S|-1 &\text{whenever $o_1 \in S$ and $|S| \ge 2$;}\\
	|S| &\text{otherwise.}
	\end{cases}\\
	&v_2(S)=|S| \qquad \qquad \forall S \subseteq O.
	\end{align*}
	It is easy to see that both $v_1$ and $v_2$ are \bsub functions. Under an allocation with bundles $A_1=\{o_1\}$ and $A_2=\{o_2,o_3,o_4\}$, $v_1(A_1)=1$ and $v_2(A_2)=3$. 
	Moreover, $v_2(A_2 \cup A_1)-v_2(A_2)=4-3=1<3=v_2(A_2)$ and $v_1(A_1 \cup A_2 \backslash\{2\})-v_1(A_1) = v_1(\{1,3,4\})-1=2-1=1=v_1(A_1)$. So, the allocation is also MEF1. 
	Agent $2$ does not envy agent $1$ since $v_2(A_1)=1<3=v_2(A_2)$; but, $v_1(A_2\backslash \{o\})=2>1=v_1(A_1)$ for every $o \in A_2$ so that agent $1$ envies agent $2$ up to more than $1$ item. Thus, the allocation is not EF1.   
	$\qed$
\end{proof}

\subsection{Approximate proportionality}\label{sec:prop}
An allocation $A$ of indivisible items $O$ to agents $N$ is said to be \textit{proportional} if $v_i(A_i) \ge \frac{1}{n} v_i(O)$ for every agent $i \in N$. Since such an allocation may not be achievable in general,\footnote{This non-existence can be demonstrated even for \bsub (in fact, binary additive) valuations by the following simple example: there are two agents and one item which each agent values at $1$. Hence, $v_1(O)=v_2(O)=1$, so that each agent must realize a value of at least $\nicefrac{1}{2}$ for the allocation to be proportional. But in any allocation, at least one agent gets no item and hence realizes zero value.} a relaxation is defined along the same lines as EF1: \textit{proportionality up to one item}. Under additive valuations,
 i.e. when $\Delta_i(S;o)=v_i(\{o\})$ for every $i \in N$, every item $o$, and every bundle $S \subseteq O \setminus \{o\}$,
 an allocation $A$ is defined to be proportional up to one item or PROP1 if, for every agent $i \in N$, there is at least one item not allocated to $i$, i.e. $o \in O \setminus A_i$, such that $v_i(A_i) \ge \frac{1}{n} v_i(O)- v_i(\{o\})$. For general non-additive valuations, it is debatable what the (subtractive) relaxation term on the right-hand side should be. For submodular valuations, we know that $\Delta_i(O \setminus \{o\};o) \le \Delta_i(S;o) \le v_i(\{o\})$ for every agent $i \in N$, every item $o \in O$, and every bundle $S \subseteq O \setminus \{o\}$; as such, we could pick either $\Delta_i(O \setminus A_i \setminus \{o\};o)$ or $v_i(\{o\})$ as the relaxation term for the strongest and weakest possible definition of PROP1 respectively. It is easy to see that, for general submodular valuations, any PO and EF1 allocation --- \emph{if it exists} --- satisfies the weakest definition of proportionality up to $1$ item (Proposition~\ref{prop:weakprop1} below); the implications of the PO+EF1 property for stronger definitions of approximate proportionality are not clear yet, even for \bsub valuations.
\begin{definition}
For general valuation functions, an allocation $A$ is called weakly proportional up to one item or WPROP1 if, for every agent $i \in N$,
\[ v_i(A_i) \ge \frac{1}{n}v_i(O) - \max_{o \in O \setminus A_i} v_i(\{o\}). \] 
\end{definition}
\begin{proposition}\label{prop:weakprop1}
	If an instance of fair allocation with indivisible goods under general submodular valuations admits a Pareto optimal and EF1 allocation, then such an allocation is also WPROP1.
\end{proposition}
\begin{proof}
	Consider a PO, EF1 allocation $A$ admitted by an instance with submodular valuations. 
	Hence, for any agent $i$, 
	\begin{align*}
	v_i(A_i) &\ge v_i(A_j \setminus \{o\}) && \text{for every $j \in N \setminus \{i\}$, for some $o \in A_j$}\\
	&= v_i(A_j) - \Delta_i(A_j \setminus \{o\};o)&&\\
	&\ge v_i(A_j) - v_i(\{o\}) && \text{since $\Delta_i(A_j \setminus \{o\};o) \le v_i(\{o\})$ due to the submodularity of $v_i(\cdot)$.}
	\end{align*}
	For every $j \in A_j$, consider an arbitrary such item $o_j$ satisfying the above inequality. Since $o_j \in A_j \subseteq O \setminus A_i$, we have $v_i(\{o_j\}) \le \max_{o \in O \setminus A_i} v_i(\{o\})$.
	Moreover, if $A_0 \neq \emptyset$, we must also have $v_i(A_i \cup A_0) = v_i(A_i)$, otherwise we could augment agent $i$'s bundle $A_i$ with the withheld set $A_0$ and thereby increase her realized valuation without diminishing that of any other agent, contradicting the Pareto optimality of $A$. Thus, combining the above $n-1$ inequalities together with this equality, we get
	\begin{align*}
	n v_i(A_i) &\ge v_i(A_i \cup A_0)+\sum_{j \in N \setminus \{i\}} v_i(A_j) - (n-1)\times \max_{o \in O \setminus A_i} v_i(\{o\})\\
	&\ge v_i(\cup_{i=0}^n A_i) - (n-1)\times \max_{o \in O \setminus A_i} v_i(\{o\}) \qquad \text{due to the submodularity, hence subadditivity, of $v_i(\cdot)$}\\
	&=v_i(O) -(n-1)\times \max_{o \in O \setminus A_i} v_i(\{o\})\\
	&\ge v_i(O) - n \times \max_{o \in O \setminus A_i} v_i(\{o\}) \qquad \text{since } \max_{o \in O \setminus A_i} v_i(\{o\}) \ge 0.
	\end{align*}
	Dividing both sides by $n$, we conclude that $A$ is WPROP1.
\end{proof}
Since a PO+EF1 allocation always exists and can be efficiently computed under \bsub valuations (Theorem~\ref{thm:bin_all}), the existence and computational tractability of a WPROP1 allocation for this function class immediately follows from Proposition~\ref{prop:weakprop1}.

\subsection{Approximate equitability}\label{sec:approx_eq}
An allocation $A$ is said to be \emph{equitable} or EQ if the realized valuations of all agents are equal under it, i.e. for every pair of agents $i,j \in N$, $v_i(A_i) = v_j(A_j)$; an allocation $A$ is \emph{equitable up to one item} or EQ1 if, for every pair of agents $i,j \in N$ such that $A_j \neq \emptyset$, there exists some item $o \in A_j$ such that $v_i(A_i) \ge v_j (A_j \setminus \{o\})$ \cite{freeman2019equitable}.\footnote{Note that if $A_j = \emptyset$ for some $j$, $v_i(A_i) \ge v_j(A_j)$ trivially. Hence the ordered pair $(i,j)$ for any $i \in N \setminus \{j\}$ could never prevent the allocation from being EQ1.} We can further relax the equitability criterion up to an arbitrary number of items: an allocation $A$ is said to be \emph{equitable up to $c$ items} or EQ$c$ if, for every pair of agents $i,j \in N$ such that $|A_j| > c$, there exists some subset $S \in A_j$ of size $|S|=c$ such that $v_i(A_i) \ge v_j (A_j \setminus S)$.\footnote{Again, if $|A_j| \le c$ for some $j$, no ordered pair $(i,j)$ for any $i \in N \setminus \{j\}$ could get in the way of the allocation being EQ$c$.} 

\citet{freeman2019equitable} show\footnote{\citet{freeman2019equitable} use an example with $3$ agents having binary additive valuations (Example 1). But it is easy to construct a fair allocation instance with only two agents having binary additive valuations that does not admit an EQ1 and PO allocation: $N=[2]$; $O=\{o_1,o_2,o_3,o_4\}$; $v_1(o)=1$ for every $o \in O$; $v_2(o_1)=1$ and $v_2(o)=o$ for every $o \in O \setminus \{o_1\}$. Obviously, any PO allocation must give $\{o_2,o_3,o_4\}$ to agent $1$ so that this agent's realized valuation is at least $2$ even after dropping one of its items; even if agent $2$ receives $o_1$, her realized valuation of $1$ will always be less than the above.} that, even for binary additive valuations (which is a subclass of the \boxs valuation class), an allocation that is both EQ1 and PO may not exist; however, in Theorem 4, they establish that it can be verified in polynomial time whether an EQ1, EF1 and PO allocation exists and, whenever it does exist, it can also be computed in polynomial time --- under binary additive valuations. We will show that the above positive result about computational tractability extends to the \boxs valuation class. We will begin by proving that  under \bsub valuations, an EQ1 and PO allocation, if it exists, is also EF1 --- we achieve this by combining Theorem~\ref{thm:eq1_bo} below with Corollary~\ref{cor:sumsquares}. This simplifies the problem of finding an EQ1, EF1 and PO allocation to that of finding an EQ1 and PO allocation. 
\begin{theorem}\label{thm:eq1_bo}
	For submodular valuations with binary marginal gains, any EQ1 and PO allocation, if it exists, is a leximin allocation.
\end{theorem} 
Hence, from Theorem~\ref{thm:convex}, we further obtain that if an EQ1 and PO allocation exists under \bsub valuations, it is also MNW and a minimizer of any symmetric strictly convex function of agents' realized valuations among all utilitarian optimal allocations. Moreover, Theorem~\ref{thm:eq1_bo}, together with Corollary~\ref{cor:bin_mnw}, implies that if an EQ1 and PO allocation exists under \bsub valuations, it must be EF1.
\begin{proof}
	Let the optimal $\USW$ for a problem instance under this valuation class be $U^*$; also, suppose this instance admits an EQ1 and PO allocation $A$. The EQ1 property implies that for every pair of agents $i,j \in N$ such that $A_j \neq \emptyset$,
	\begin{align*}
	v_i(A_i) &\ge v_j(A_j \setminus \{o\}) \qquad \text{for some $o \in A_j$}\\
	 &=v_j(A_j) - \Delta_j (A_j \setminus \{o\}; o)\\
	&\ge v_j(A_j) - 1, \qquad \text{since $\Delta_j (A_j \setminus \{o\}; o) \in \{0,1\}$.}
	\end{align*}
 This inequality holds trivially and strictly if $A_j = \emptyset$. Thus, $\max_{i \in N} v_i(A_i) \le \min_{i \in N} v_i(A_i)+1$. In other words, there exist a non-negative integer $\alpha \le U^*$ and a positive integer $n_0 \in [n]$ such that $n_0$ agents have valuations $\alpha$ each and the remaining agents, if any, have valuations $\alpha+1$ each under allocation $A$, with $U^*=n_0 \alpha + (n-n_0)(\alpha+1)=n\alpha + n-n_0$. We can write the agents' realized valuations under $A$ (with arbitrary tie-breaking) as the $n$-dimensional vector
	\[\score(A)=\left(\underbrace{\alpha, \alpha, \dots, \alpha}_\text{$n_0$ entries}, \underbrace{\alpha+1, \alpha+1, \dots, \alpha+1}_\text{$n-n_0$ entries}\right).\]
	If $A$ were not leximin, there would be another allocation $A'$ for which the corresponding valuation vector $\score(A')$ would have an entry strictly higher than that of $A$ at the same position, say $n' \in [n]$. If $n' \le n_0$, then every entry of $\score(A')$ from position $n'$ is at least $\alpha+1$, so the $\USW$ under $A'$ is
	\begin{align*}
	U' &\ge (n'-1)\alpha + (n-n'+1)(\alpha+1)\\
	&=n\alpha + n-n'+1\\
	&\ge n\alpha +n-n_0+1\\
	&=U^*+1.
	\end{align*}
	If $n_0 < n' \le n$, then similarly,
	\begin{align*}
	U' &\ge n_0 \alpha + (n' - n_0-1) (\alpha+1) + (n-n'+1)(\alpha+2)\\
	&= n\alpha +2n -n' -n_0 +1\\
	&\ge n\alpha +2n -n -n_0 +1\\
	&=n\alpha +n -n_0 +1\\
	&=U^*+1.
	\end{align*}
	In either case, we have a contradiction since $U^*$ is the optimal utilitarian social welfare for this instance. Hence, $A$ must be leximin.
\end{proof}
We conjecture that a stronger result holds: under \bsub valuations, the leximin allocation is optimally EQ$c$ for $c \in \{0, 1, \dots, m\}$ among all PO allocations. A proof or a counterexample remains elusive. We present this more formally as follows:
\begin{conjecture}\label{thm:opteqc}
		For a problem instance where all agents have submodular valuations with binary marginal gains, if $c^*$ is the smallest $c \in [m]$ for which a leximin allocation is EQ$c$, then the instance admits no PO allocation that is EQ$c$ for any $c < c^*$.
\end{conjecture}

\subsection{The MMS guarantee}\label{sec:mms}
Let $\Pi(O)$ denote the collection of all $n$-partitions of the set of items $O$. 
Then, the \emph{maximin share} of agent $i$  is defined as: $$\MMS_i \triangleq \max_{(A_1,A_2,\dots,A_n) \in \Pi(O)} \min_{j \in N} v_i(A_j).$$ 
An allocation $A$ is called MMS if $v_i(A_i) \ge \MMS_i$ for every $i \in N$; for any approximation ratio $\alpha \in (0,1]$, $A$ is called $\alpha$-MMS if $v_i(A_i) \ge \alpha \cdot \MMS_i$ for every $i \in N$.

\citet{ghodsi2018fair} showed that for agents with (general) submodular valuations, a $\nicefrac{1}{3}$-MMS allocation always exists and can be computed in polynomial time; it is not yet clear what the connections of the MMS guarantee with PO+EF1 allocations are, even for the subclass of submodular valuations with binary marginal gains, and whether a better approximation ratio can be achieved for this subclass.  

In general, a PO and EF1 allocation may not be MMS even for \boxs valuations, which we show by an example below:
\begin{proposition}\label{prop:poef1_notMMS}
	For assignment valuations (hence submodular valuations) with binary marginal gains, a Pareto optimal and EF1 allocation may not be MMS.
\end{proposition}
\begin{proof}
	Consider an augmentation to Example~\ref{runningex} from Section~\ref{sec:utiloptef1}: we now have eight items and two groups $N_1=\{a_1,a_2, \dots, a_6\}$ and $N_2=\{b_1,b_2,b_3,b_4\}$ with \boxs valuations; an  agent's utility for an item is $1$ if there is an edge between the nodes representing them in Figure~\ref{fig3}, otherwise it is $0$.  
	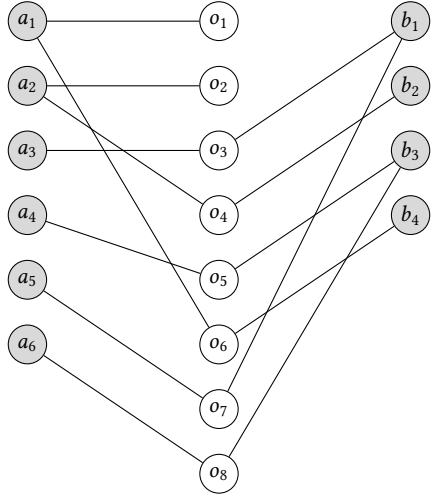
\begin{figure}[htb]
		\centering
		\begin{tikzpicture}[scale=0.85, transform shape, every node/.style={minimum size=6mm, inner sep=1pt}]
		\node[draw, circle,fill=gray!30](1) at (0,2) {$a_1$};
		\node[draw, circle,fill=gray!30](2) at (0,1) {$a_2$};
		\node[draw, circle,fill=gray!30](3) at (0,0) {$a_3$};
		\node[draw, circle,fill=gray!30](4) at (0,-1) {$a_4$};
		\node[draw, circle,fill=gray!30](5) at (0,-2) {$a_5$};
		\node[draw, circle,fill=gray!30](6) at (0,-3) {$a_6$};
		
		\node[draw, circle](11) at (3,2) {$o_1$};
		\node[draw, circle](12) at (3,1) {$o_2$};
		\node[draw, circle](13) at (3,0) {$o_3$};
		\node[draw, circle](14) at (3,-1) {$o_4$};
		\node[draw, circle](15) at (3,-2) {$o_5$};
		\node[draw, circle](16) at (3,-3) {$o_6$};
		\node[draw, circle](17) at (3,-4) {$o_7$};
		\node[draw, circle](18) at (3,-5) {$o_8$};
		
		\node[draw, circle,fill=gray!30](21) at (6,2) {$b_1$};
		\node[draw, circle,fill=gray!30](22) at (6,1) {$b_2$};
		\node[draw, circle,fill=gray!30](23) at (6,0) {$b_3$};
		\node[draw, circle,fill=gray!30](24) at (6,-1) {$b_4$};
		
		\draw[-, >=latex] (1)--(11) (1)--(16);
		\draw[-, >=latex] (2)--(12) (2)--(14);
		\draw[-, >=latex] (3)--(13);
		\draw[-, >=latex] (4)--(15);
		\draw[-, >=latex] (5)--(17);
		\draw[-, >=latex] (6)--(18);
		
		\draw[-, >=latex] (21)--(13);
		\draw[-, >=latex] (21)--(17);
		\draw[-, >=latex] (22)--(14);
		\draw[-, >=latex] (23)--(15);
		\draw[-, >=latex] (23)--(18);
		\draw[-, >=latex] (24)--(16);
		\end{tikzpicture}
		\caption{Problem instance used in the proof of Proposition~\ref{prop:poef1_notMMS}.\label{fig3}}
	\end{figure}
	The maximin share of group/agent $N_2$ is $\MMS_2=3$ corresponding to the partition $\{A_1,A_2\}$ where $A_1 \supseteq \{o_3,o_4,o_5\}$, $A_2 \supseteq \{o_6,o_7,o_8\}$, and each of $o_1,o_2$ is arbitrarily included in either $A_1$ or $A_2$: with $A_1$ (resp. $A_2$) allocated to $N_2$, the members $b_1,b_2,b_3$ (resp. $b_1,b_3,b_4$) are assigned items $o_3,o_4,o_5$ (resp. $o_7,o_8,o_6$) respectively.
	
	However, take the allocation $A'$ with bundles $A'_2=\{o_4,o_6\}$ and $A'_1 = O \setminus A'_2$: $v_1(A'_1)=6$ with $o_1,o_2,o_3,o_5,o_7,o_8$ assigned to $a_1,a_2,a_3,a_4,a_5,a_6$ respectively; $v_2(A'_2)=2$ with $o_4,o_6$ assigned to $b_2,b_4$ respectively; $v_1(A'_2)=2$ with $o_4,o_6$ assigned to $a_2,a_1$ respectively; $v_2(A'_1)=2$ with $o_1,o_2$ unassigned, either $o_3$ or $o_7$ assigned to $b_1$, and either $o_5$ or $o_8$ assigned to $b_3$ respectively. Since $v_1(A_1)+v_2(A_2)=8=|O|$, the allocation is utilitarian optimal (hence PO); it is also envy-free since $v_1(A'_1) > v_1(A'_2)$ and $v_2(A'_2) = v_2(A'_1)$. But $v_2(A'_2) < \MMS_2$.
\end{proof}	
There are a few follow-up questions that are yet to be settled:
\begin{itemize}
	\item Does an MNW/leximin allocation under \bsub valuations (Theorem~\ref{thm:convex} and Corollary~\ref{cor:bin_mnw}) have the MMS guarantee?\footnote{Notice that, in Example~\ref{runningex}, $A_1=\{o_1,o_2,o_3\}$ and $A_2 = \{o_4,o_5,o_6\}$ constitute the unique MNW/leximin allocation that is also MMS.}
	\item Is a PO and EF1 allocation under \bsub valuations guaranteed to provide a constant approximation to the MMS guarantee, and if so, what is a tight bound on this approximation ratio? 
\end{itemize}
We conclude with a positive existential result on a subclass of \boxs valuations: binary additive valuations, i.e. for every $i \in N$, we have $v_i(\{o\}) \in \{0,1\}$ for every $o \in O$, and $v_i(S) = \sum_{o \in S} v_i(\{o\})$ for every $S \subseteq O$. For this valuation subclass, we can safely assume that there are no \textit{redundant} items in $O$, i.e. there is no $o \in O$, such that $v_i(\{o\})=0$ for every $i \in N$. \footnote{Proposition~\ref{prop:MMS} also follows from \citet{segal2019democratic} who extend many of the fairness notions we study in this paper to a setting where agents are partitioned in $k$ groups (ours is a special case with $k=n$ since each agent can be viewed as constituting its own group): their Lemmas 2.3 and 2.7(c) jointly imply that under binary additive valuations, every complete, EF1 allocation is MMS. We thank an anonymous reviewer for pointing us to this connection.}
\begin{proposition}\label{prop:MMS}
 If all agents have binary additive valuations, any Pareto optimal, EF1 allocation is also MMS.	
\end{proposition}
\begin{proof}
	It is easy to see that a PO allocation must be complete for this valuation class: if not, we would have an item $o^* \in A_0$ and, by our non-redundancy assumption, an agent $i^*$ with $v_{i^*}(\{o\})=1$, hence we could improve agent $i^*$'s valuation at no cost to the other agents by allocating the item $o^*$ to agent $i^*$. We will prove the required result by contraposition. Let $A$ be an allocation that is PO but not MMS for an instance where all agents have binary additive valuations. Then, there is an agent $i$ such that $v_i(A_i) < \mu$ where $\mu = \MMS_i$. Since valuations are integers, $v_i(A_i) \le \mu -1$.
	
	 If $v_i(A_j) \le \mu$ for every other agent $j \in N \setminus \{i\}$, then due to additivity and the emptiness of $A_0$, 
	\[v_i(O)= v_i(A_i) + \sum_{j \in N \setminus \{i\}} v_i(A_j) = v_i(A_i) + \sum_{j \in N \setminus \{i\}} v_i(A_j) < \mu + (n-1) \mu = n  \mu.\]
	
	
	But, from the definition of the maximin share, there must exist a complete allocation $A'$ such that $v_i(A'_j) \ge v_i(A'_i)=\mu$ for every $j \in N \setminus \{i\}$. This, together with additivity, implies that
	\[v_i(O) = \sum_{j \in N} v_i(A'_j) \ge n \mu, \]
	a contradiction. Hence, under the allocation $A$, there must be at least one agent, say $k \in N\setminus \{i\}$, such that $v_i(A_k) > \mu$.
	
	Thus, for any item $o \in A_k$, we have 
	\[ v_i(A_k \setminus \{o\}) = v_i(A_k) - v_i(\{o\}) \ge v_i(A_k) -1 > \mu -1 > v_i(A_i). \]
	In other words, agent $i$ is not envy-free of agent $k$ up to $1$ item, i.e. $A$ cannot be EF1.
	
	Hence, any PO and EF1 allocation must also be MMS.
\end{proof}

\end{document}